\documentclass[journal]{IEEEtran}
\usepackage{amssymb}
\usepackage{amsmath}
\usepackage{dcolumn}
\usepackage{bm}
\usepackage{color}
\usepackage{graphicx}
\usepackage{subfigure} 
\usepackage{tabularx}
\usepackage{array}
\usepackage{amsthm}
\usepackage{cite}
\usepackage{hyperref}
\usepackage{enumerate}

\newtheorem{theorem}{Theorem}
\newtheorem{corollary}{Corollary}
\newtheorem{lemma}{Lemma}
\newtheorem{remark}{Remark}
\newtheorem{definition}{Definition}
\newtheorem{assumption}{Assumption}
\newtheorem{proposition}{Proposition}
\newtheorem{problem}{Problem}

\renewcommand{\t}{^{\mbox{\tiny\sf T}}}
\newcommand{\bremark}{\begin{remark}
\begin{rm}}
\newcommand{\eremark}{ \end{rm}\hfill \rule{1mm}{2mm}
\end{remark} }
\newcommand{\btheorem}{\begin{theorem} \begin{it}}
\newcommand{\etheorem}{\end{it} \hfill \rule{1mm}{2mm}
\end{theorem} }
\newcommand{\blemma}{\begin{lemma} \begin{it} }
\newcommand{\elemma}{ \end{it} \hfill\rule{1mm}{2mm}
\end{lemma} }
\newcommand{\bcorollary}{\begin{corollary} \begin{it} }
\newcommand{\ecorollary}{ \end{it} \hfill\rule{1mm}{2mm}
\end{corollary} }
\newcommand{\bdefinition}{\begin{definition} }
\newcommand{\edefinition}{ \hfill\rule{1mm}{2mm}
\end{definition} }
\newcommand{\bproposition}{\begin{proposition} }
\newcommand{\eproposition}{\hfill \rule{1mm}{2mm}
\end{proposition} }
\newcommand{\bexample}{\begin{example} \begin{rm}}
\newcommand{\eexample}{ \end{rm} \hfill\rule{1mm}{2mm}
\end{example} }

\newcommand{\bassumption}{\begin{assumption} }
\newcommand{\eassumption}{\hfill \rule{1mm}{2mm}
\end{assumption} }

\newcommand{\balgorithm}{\medskip\begin{algorithm} \rm}
\newcommand{\ealgorithm}{ \hfill \rule{1mm}{2mm}\medskip
\end{algorithm} }

\newcommand{\basm}{\begin{assumption} \begin{rm} }
\newcommand{\easm}{ \end{rm} \hfill\rule{1mm}{2mm}
\end{assumption} }

\begin{document}

\title{Coordinated Navigation Control of Cross-Domain Unmanned Systems via Guiding Vector Fields }

\author{
\IEEEauthorblockN{Bin-Bin Hu, Hai-Tao Zhang\IEEEauthorrefmark{1}, Bin Liu, Jianing Ding, Yifan Xu, 
Chuanshang Luo
and 
Haosen Cao
}
\thanks{This work was supported by the National key R\&D Program of China under Grant 2022ZD0119601, in part by the National Natural Science
Foundation of China under Grants 62225306, U2141235, 52188102 and 62003145, in part by the Guangdong Basic and Applied Research Foundation under Grant 2022B1515120069 ({\it Corresponding author: H.-T. Zhang}).
}
\thanks{Bin-Bin Hu, Hai-Tao Zhang, Jianing Ding, Yifan Xu, Chuanshang Luo, and Haosen Cao are with the School of Artificial Intelligence and Automation, the Engineering Research Center of Autonomous Intelligent Unmanned Systems, the Key Laboratory of Image Processing and Intelligent Control, and the State Key Lab of Digital Manufacturing Equipment and Technology,
Huazhong University of Science and Technology, Wuhan 430074, P.R.~China
(e-mails: hbb@hust.edu.cn, zht@mail.hust.edu.cn, djn@hust.edu.cn, m202173 254@hust.edu.cn, csluo@hust.edu.cn, chs@hust.edu.cn).
}
\thanks{Bin Liu is with Dongguan University of Technology, Dongguan, 523106, P.R.~China  
(e-mail: liubin@dgut.edu.cn).
}

}

\maketitle

\begin{abstract}
This paper proposes a distributed guiding-vector-field (DGVF) controller for cross-domain unmanned systems (CDUSs) consisting of heterogeneous unmanned aerial vehicles (UAVs) and unmanned surface vehicles (USVs), to achieve coordinated navigation whereas maneuvering along their prescribed paths. In particular, the DGVF controller provides a hierarchical architecture of an upper-level heterogeneous guidance velocity controller and a lower-level signal tracking regulator. Therein, the upper-level controller is to govern multiple heterogeneous USVs and UAVs to approach and maneuver along the prescribed paths and coordinate the formation simultaneously, whereas the low-level regulator is to track the corresponding desired guidance signals provided by the upper-level module. Significantly, the heterogeneous coordination among neighboring UAVs and USVs is achieved merely by the lightweight communication of a scalar (i.e., the additional virtual coordinate), which substantially decreases the communication and computational costs. Sufficient conditions assuring asymptotical convergence of the closed-loop system are derived in presence of the exponentially vanishing tracking errors. Finally, real-lake experiments are conducted on a self-established cross-domain heterogeneous platform consisting of three M-100 UAVs, two HUSTER-16 USVs, a HUSTER-12C USV, and a WiFi 5G wireless communication station to verify the effectiveness of the present DGVF controller.  
\end{abstract}
 
\begin{IEEEkeywords}
Cross-domain coordination, path navigation, unmanned aerial vehicles (UAVs), unmanned surface vehicles (USVs), guiding vector fields
\end{IEEEkeywords}

\IEEEpeerreviewmaketitle

\section{Introduction}
Coordinated navigation of unmanned systems, including various kinds of vehicles, such as unmanned aerial vehicles (UAVs), unmanned surface vehicles (USVs), unmanned ground vehicles (UGVs), and unmanned underwater vehicles (UUVs) \cite{yao2021singularity}, plays an increasingly essential role in robotics community due to its superiority in large operational range, high efficiency, strong robustness, and superb flexibility. Such virtues lay a solid foundation for various applications, such as intelligent transportation, convoying \& rescue, environmental exploration, and surveillance \& patrolling~\cite{hu2023cooperative, hu2021bearing}. 

Notably, initial efforts were mainly devoted to the coordination of single-domain unmanned systems.  As a typical marine intelligent vehicle, USV plays an important role in conducting abundant marine operations, where the coordinated navigation of USVs has been intensively explored in recent years. For instance, for coordinated path-following tasks, a line-of-sight (LOS) guidance method with a back-stepping technique was proposed in \cite{liu2018cooperative} to follow closed paths. Other relevant methods concerning neurodynamic optimization \cite{peng2018path} and reinforcement learning \cite{zhao2021usv} were developed to cope with state constraints and unknown braking disturbances. For more complex coordinated target surrounding tasks, an output-regulation-based controller was developed in \cite{liu2019collective} for static targets, which was afterward extended to motional targets with variational velocities \cite{hu2021distributed2,hu2020multiple}. For other specific marine operations like scanning-chain navigation, a path-following algorithm with 2D B-spline curves was designed in \cite{liu2020scanning}. A recent work proposed a guiding-vector-field controller in \cite{hu2023spontaneous} for USVs to form a {\it spontaneous-ordering} platoon whereas performing path navigation. However, most of the existing studies \cite{liu2018cooperative,peng2018path,zhao2021usv,liu2019collective,hu2021distributed2,hu2020multiple,liu2020scanning} still focused on single-domain coordination, namely, a 2D horizontal water surface, which inevitably hinders their further applications due to the limited detectable ranges and operational regions.

Fortunately, UAVs endow  3D coverage and high flexibility in sky detection, where the associated coordinated navigation has also been extensively explored in the literature these years. 
A decentralized navigation scheme was proposed in \cite{huang2021decentralized} for UAV swarms to achieve aerial surveillance. A minimum-control (MINCO) optimization framework was developed in~\cite{zhou2022swarm} for micro UAV swarms in the wild. Thereafter, it was extended to a distributed optimization framework with the assistance of differentiable graphs to yield formations in dense environments \cite{quan2022formation}. A novel payload-based controller was proposed in \cite{salinas2023unified} to cooperatively transport a slung load in forward flight. 
To address more complex aerodynamic interaction forces in swarms, a deep learning neural networks (DNNs)-based method entitled Neural-Swarm2 was proposed in \cite{shi2021neural} to coordinate heterogeneous multirotors in close proximity. Unfortunately, most of these existing UAV-coordination works \cite{zhou2022swarm,quan2022formation, shi2021neural,salinas2023unified} rely on high communication and computational costs because of the online optimization and learning process, which may not be suitable to the lightweight on-board controller in practice.  A more practical method is thus motivated by recent work \cite{yao2022guiding} only utilizing guiding vector fields (GVFs) with an additional virtual coordinate, where the motion coordination is achieved by communicating with lightweight virtual coordinates.

Despite the coordination of USVs or UAVs performing well in some operations, there still exists some inherent weaknesses. For example, UAVs only carry compact batteries, which thus cannot sustain long-term execution. Note that the detectable range of USVs is constrained by water surfaces. However, if the USV fleets are equipped with onboard UAV swarms, the detection region of the cross-domain group is substantially increased from 2D to 3D, which can unlock much more advanced applications, such as maritime-air reconnaissance, detection, rescue, anti-smuggling, etc. Promisingly, these years have witnessed some works exploring UAV-USV coordination in the literature, especially the cooperative landing task of a UAV on a motional USV~\cite{shao2019novel}. In this pursuit, a sliding-mode control scheme was developed in~\cite{lee2018sliding} for motional UAV-USV landing. Later, it was improved to a fixed-time landing approach with external disturbances~\cite{xia2020adaptive}. For more complex visual-integrated autonomous landing, the feasibility of such a visual navigation system was verified by both numerical simulations~\cite{meng2019vision} and real-lake experiments~\cite{zhang2021visual}. However,  due to the arduousness of cross-domain multiple UAV-USV cooperation, most of the existing landing works~\cite{shao2019novel,lee2018sliding,xia2020adaptive,meng2019vision,zhang2021visual} still focused on the single-UAV single-USV situation.  
As representative works of the few investigation on much more challenging multiple UAV-USV coordination, the initial efforts have been devoted to the fundamental theory of the UAV-USV fleets in \cite{su2018controllability,long2018group}, where the USVs and UAVs were regarded as a two-time-scale system, namely, a slow-UAV subsystem and a fast-USV subsystem. Then, a distributed-model-predictive-control (DMPC) scheme was proposed in~\cite{huang2020formation} to produce a coordinated formation of multiple UAV-USV fleets. A distributed fixed-time control protocol was proposed in \cite{liu2023distributed}  to stabilize heterogeneous UAV-USV formations. A flexible leader-follower formation controller was developed in~\cite{simonsen2020application} for autonomous multiple UAV-USV navigation. Thereafter, for more complicated situations, a cooperative 3D-mapping guidance signal with adaptive fuzzy control was designed in \cite{li2022novel} to address a UAV-USV path-following problem in presence of both system uncertainties and external disturbances. 

 Till date, most of the aforementioned multiple UAV-USV coordination works~\cite{su2018controllability,long2018group,huang2020formation,liu2023distributed,simonsen2020application,li2022novel} are only verified by numerical simulations, which have not touched real marine experiments due to the practical challenges in the cross-domain coordination, such as the wind/wave/tide disturbances, limited sensor regions, underactuation and saturated actuators. Additionally, the aforementioned works \cite{su2018controllability,long2018group,huang2020formation,liu2023distributed,simonsen2020application,li2022novel} coordinate the UAV-USV formation via interacting complex Euclidean distance among vehicles, which will inevitably increase communication, sensor and computational costs, and thereby pose new challenges in real applications. Therefore, it becomes an urgent yet challenging mission to develop a more efficient and practical scheme to achieve realistic multiple UAV-USV coordinated navigation.

To this end, inspired by the GVFs with an additional virtual coordinate in \cite{yao2022guiding}, we propose a distributed GVF (in abbr.~DGVF) controller for heterogeneous CDUSs to achieve coordinated navigation whereas maneuvering along the prescribed paths. Particularly, the DGVF controller is established as a hierarchical architecture to decouple heterogeneous CDUSs, which consists of an upper-level heterogeneous guidance velocity controller and a lower-level signal tracking regulator. More precisely, the upper-level  controller is to govern USVs and UAVs to converge to and maneuver along their prescribed paths and coordinate their formation simultaneously. By contrast, the low-level regulator is to track the corresponding desired guidance signals. In this way, we establish a cross-domain heterogeneous platform consisting of three M-100 UAVs, two HUSTER-16 USVs, a HUSTER-12C USV, and a WiFi 5G wireless communication station on Songshan Lake, Dongguan City, Guangdong Province, China. Finally, real-lake experiments using such a cross-domain heterogeneous platform  to verify the effectiveness of the proposed method. In summary, the main contribution of this paper is three-fold.

\begin{enumerate}

\item We develop a DGVF controller for the heterogeneous CDUS to converge to and maneuver along their prescribed paths while coordinating their formations.

\item  We design a lightweight communication-based protocol (i.e., exchanging their virtual coordinates), which substantially reduces real-time communication, sensor, and computational costs for complex heterogeneous coordination.

\item We establish a cross-domain heterogeneous USV-UAV formation coordination platform with three M-100 UAVs, two HUSTER-16 USVs, a HUSTER-12C USV, and a WiFi 5G wireless communication station. 

\end{enumerate}

The rest of the paper is organized as follows. Section~II introduces some preliminaries and formulates the problem. Section III proposes the UAV-USV formation coordination controller, and elaborates on the main technical results.
Afterwards, real lake-based experiments and complex simulations  are conducted in Section IV to verify the effectiveness of the proposed controller. Finally, the conclusion is drawn in Section~V.

{\it Notations:} The real numbers and positive real numbers are denoted by $\mathbb{R},\mathbb{R}^+$, respectively. The $n$-dimensional Euclidean space is denoted by $\mathbb{R}^n$. The integer is represented by $\mathbb{Z}$, where the set of integer $\{m\in \mathbb{Z}~|~i\leq m\leq j\}$ is represented by $\mathbb{Z}_i^j$. The Euclidean norm of a vector $a$ is $\|a\|$. The transpose of a matrix $A$ is represented by $A\t$. The $n$-dimensional identity matrix is represented by $I_n$. The $n$-dimensional column vector consisting of all 1's is denoted by $\mathbf{1}_n$.

\section{Preliminaries}
\subsection{Path-Following Vector Fields}
Suppose a prescribed path $\mathcal P^{phy}$ living in the $n$-dimensional Euclidean space is defined by the zero-level set
of implicit functions,
\begin{align}
\label{implicit_path_definition}
\mathcal P^{phy}:=\{ \sigma\in\mathbb{R}^n~|~\phi_i(\sigma)=0, i=1,\dots, n-1\},
\end{align}
where $\sigma\in\mathbb{R}^n$ are the coordinates and $\phi_i(\cdot): \mathbb{R}^n\rightarrow\mathbb{R}$ is a twice continuously differentiable function, i.e., $\phi_i(\cdot)\in C^2$.

The prescribed path $\mathcal P^{phy}$ in \eqref{implicit_path_definition} is a topological description of 1D connected differential manifold~\cite{freedman2014topology}, where the functions $\phi_i(q_0), i=1, \dots, n-1,$ provide a simple method to measure the distance implicitly between a point $q_0\in\mathbb{R}^n$ and the prescribed path  $\mathcal P^{phy}$ instead of the sophisticated calculation of $\mbox{dist}(q_0,\mathcal P^{phy}):=\mbox{inf}~\{\|p-q_0\|~|~p\in\mathcal P^{phy}\}$. For instance, a circular path in the 2D plane is described by $\phi(\sigma_1, \sigma_2)=\sigma_1^2+\sigma_2^2-r^2=0$ with $\sigma_1, \sigma_2\in\mathbb{R}$ being the coordinates and $r\in\mathbb{R}^+$ the radius. Then, for arbitrary point $q_0\in\mathbb{R}^2$, $\phi_i(q_0)$ can be utilized as the signed circular path-following errors implicitly, where $|\phi_i(q_0)|$ decreases as the point $q_0$ goes near the circular path, and $|\phi_i(q_0)|=0$ if $q_0$ is on the path. However, there still exist some pathological situations (i.e., $\phi_i(q_0)=0$ does not necessarily mean $\mbox{dist}(q_0,\mathcal P)=0$) if the functions $\phi_i(\cdot)$ are not properly selected, which are excluded by the following assumption.

\begin{assumption}
\cite{yao2021singularity}
\label{assp_error} For any given $\kappa>0$, an arbitrary point $q_0\in\mathbb{R}^n$ and a desirable physical path $\mathcal{P}^{phy}$, we assume that
$$ \inf\big\{ |\phi_i(q_0)|: \mbox{dist}(q_0,\mathcal P^{phy})\geq \kappa\big\}>0.$$
\end{assumption}
Assumption~\ref{assp_error} selects a valid function $\phi(\cdot)$ to guarantee $ \lim_{t\rightarrow\infty} |\phi_i(q_0(t))|=0\Rightarrow\lim_{t\rightarrow\infty}\mbox{dist}(q_0(t),\mathcal P^{phy})=0$, where $q_0(t)$ is a trajectory in the $n$-dimensional Euclidean space.  Note that Assumption~\ref{assp_error} can be easily satisfied with some polynomial or trigonometric functions (see e.g., \cite{yao2021singularity}).

According to the description of $\mathcal P^{phy}$ in \eqref{implicit_path_definition} and Assumption~\ref{assp_error}, the guiding vector field (GVF) $\chi^{phy}$ for the path-following problem is stated below (see, e.g., \cite{yao2021singularity}),
\begin{align}
\label{eq_GVF}
\chi^{phy}=&\times (\nabla\phi_1(q), \cdots, \nabla\phi_{n-1}(q))\nonumber\\
&-\sum_{j=1}^{n-1}k_{j}\phi_j(q)\nabla\phi_j(q),
\end{align}
where the symbol $\times$ represents the cross product, $k_j$ the control gain, $q\in\mathbb{R}^n$ the position of the vehicle, and $\nabla\phi_j(\zeta): \mathbb{R}^{n}\rightarrow\mathbb{R}^{n}$ the gradient of $\phi_j$ w.r.t. $\zeta$.
The former term $\times (\nabla\phi_1(q), \cdots, \nabla\phi_{n-1}(q))$ in Eq.~\eqref{eq_GVF}  perpendicular to the gradients $\nabla\phi_j(\xi), j\in \mathbb{Z}_1^{n-1}$, provides a propagation velocity for the vehicle to maneuver along the prescribed path $\mathcal P^{phy}$, whereas the latter term $-\sum_{j=1}^{n-1}k_{j}\phi_j(q)\nabla\phi_j(q)$ is the gradient to guide the vehicle to converge to the prescribed path $\mathcal P^{phy}$. With the aforementioned designed $\chi^{phy}$ in \eqref{eq_GVF}, the vehicle will finally approach and move along the prescribed path.

However, there may exist still some singular points (i.e., $\chi^{phy}=\mathbf{0}_n$) that undermine the global convergence of the GVF $\chi^{phy}$ in Eq.~\eqref{eq_GVF}. 
To eliminate such undesirable singular points and guarantee a global convergence to the prescribed path~$\mathcal P^{phy}$, a higher-dimensional GVF is introduced with an additional virtual coordinate $\omega\in \mathbb{R}$.

\begin{definition}
\label{definition_GVF}
(Higher-dimensional GVF) \cite{yao2021singularity} Suppose a prescribed path $\mathcal P^{phy}$ is parameterized by 
\begin{align}
\label{HD_path_phy}
\mathcal P^{phy}:=\big\{[\sigma_1, \cdots, \sigma_n]\t\in\mathbb{R}^n~|~\sigma_j=f_j(\omega), j\in\mathbb{Z}_1^n\big\}
\end{align}
with the $j$-th coordinate $\sigma_j\in\mathbb{R}$, the virtual coordinate parameter $\omega\in\mathbb{R}$, and the twice continuously differentiable function $ f_j \in C^2$. 
Define $\xi=[\sigma_1,\dots, \sigma_n, \omega]\t\in \mathbb{R}^{n+1}$ as an augmented coordinate vector, and one can establish a corresponding higher-dimensional path 
\begin{align}
\label{HD_path_hgh}
\mathcal P^{hgh}:=\big\{\xi\in\mathbb{R}^{n+1}~|~ \phi_j(\xi)=0, j\in\mathbb{Z}_1^n\big\}
\end{align}
with the implicit functions $\phi_{j}(\xi):=\sigma_j-f_j(\omega)\in\mathbb{R}$. 
\begin{figure}[!htb]
\centering
\includegraphics[width=5cm]{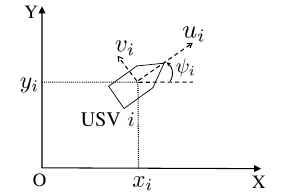}
\caption{ Illustration of the USV dynamics.}
\label{USV_dynamic}
\end{figure}
Let the position of vehicle be $q\in \mathbb{R}^n$ and defining $\zeta=[q\t, \omega]\in\mathbb{R}^{n+1}$ as the augmented position vector, we then substitute $\zeta\in\mathbb{R}^{n+1}$ into $\xi\in\mathbb{R}^{n+1}$ of Eq.~\eqref{HD_path_hgh} to yield a higher-dimensional GVF $\chi^{hgh}\in\mathbb{R}^{n+1}$ 
\begin{align}
\label{high_eq_GVF1}
\chi^{hgh}=&\times (\nabla\phi_1(\zeta), \cdots, \nabla\phi_n(\zeta))-\sum_{j=1}^nk_{j}\phi_j(\zeta)\nabla\phi_j(\zeta),\nonumber\\
                =&\begin{bmatrix}
        	              (-1)^n \partial{f}_{1}(\omega)-k_{1}\phi_{1}(\zeta) \\
	              \vdots\\
                      (-1)^n\partial{f}_{n}(\omega)-k_{n}\phi_{n}(\zeta)\\
                      (-1)^n+\sum\limits_{j=1}^nk_{j}\phi_{j}(\zeta)\partial{f}_{j}(\omega)\\
                   \end{bmatrix},  
\end{align}
where $\nabla\phi_j(\zeta):=[0,\dots,1,\dots,-\partial{f}_{j}(\omega) ]\t\in\mathbb{R}^{n+1}, j\in\mathbb{Z}_1^n$ is the gradient of $\phi_j(\zeta)$ w.r.t. the vector $\zeta$, $\partial{f}_{j}:={\partial f_{j}(\omega)}/{\partial\omega}$ denotes the partial derivative of $f_{j}(\omega)$ w.r.t. the virtual coordinate $\omega$, and $k_j$ is the control gain.  
\end{definition}

Definition~\ref{definition_GVF} illustrates the non-zero feature of $\chi^{hgh}$ (i.e., $\chi^{hgh}\neq\mathbf{0}_{n+1}$) due to the non-zero term $(-1)^n$ in the last row of Eq.~\eqref{high_eq_GVF1}, 
which hence eliminates the singular points of the original GVF $\chi^{phy}$ in \eqref{eq_GVF}. By projecting $\chi^{hgh}$ into the first $n$-dimensional Euclidean space, the vehicle will approach and then move along the prescribed path $\mathcal P^{phy}$ from any initial position.

\subsection{Cross-Domain Unmanned Systems}
We consider a CDUS consisting of $n$ USVs represented by ${\cal V}_1=\{1,2,\dots, n\}$ and $m$ UAVs represented by ${\cal V}_2=\{n+1, n+2,\dots, n+m\}$.
As shown in Fig.~\ref{USV_dynamic}, the kinematics of the USV $i, i \in{\cal V}_1$ in the Cartesian coordinates are described below \cite{hu2021bearing},
\begin{align}
\label{USV_kinetic}
 \dot{q}_{i,x} &=u_i\cos\psi_i-v_i\sin\psi_i, \nonumber\\
\dot{q}_{i,y} &=u_i\sin\psi_i+v_i\cos\psi_i, \nonumber\\
\dot{\psi}_i &=r_i,
\end{align}
where $q_i(t):=[q_{i, x}(t), q_{i, y}(t)]\t\in \mathbb{R}^2$ denotes the position, $\psi_i(t) \in \mathbb{R}$ the yaw angle, and $u_i(t), v_i(t), r_i(t) \in \mathbb{R}$ the surge, sway and yaw velocities of USV $i$ in the body coordinate, respectively.

The dynamics of USV $i, i \in{\cal V}_1$ are generally described by a practical model (see e.g., \cite{liu2019collective}) below
\begin{align}
\label{USV_kinematic}
\dot{u}_i &=\epsilon_1u_i+\epsilon_2v_i r_i+\epsilon_3\tau_{i,u}, \nonumber\\
\dot{r}_i &=\epsilon_4r_i+\epsilon_5\tau_{i,r}, \nonumber\\
\dot{v}_i &=\epsilon_6v_i+\epsilon_7u_i r_i,
\end{align}
where $\epsilon_1, \epsilon_2, \epsilon_3, \epsilon_4, \epsilon_5, \epsilon_6, \epsilon_7\in \mathbb{R}$ are the identified parameters via zig-zag experiments, and $\tau_{i,u}, \tau_{i,r}\in \mathbb{R}$ the actuator inputs of USV~$i, i\in\mathcal V_1$.

For the subgroup of the UAVs,  the kinematics of the UAV $j, j\in{\cal V}_2$ are generally approximately illustrated by a second-order integrator by ignoring the fast-time-scale inner-loop motor regulator \cite{dong2014time}, 
\begin{align}
\label{UAV_kinetic}
 \dot{q}_j &=p_j, \nonumber\\
 \dot{p}_j &=\tau_j, 
\end{align}
where $q_j=[q_{j, x}, q_{j, y}, q_{j, z}]\t\in\mathbb{R}^3, p_j=[p_{j, x}, p_{j, y}, p_{j, z}]\t\in\mathbb{R}^3, \tau_j=[\tau_{j, x}, \tau_{j, y}, \tau_{j, z}]\t\in\mathbb{R}^3$ are the position, velocity and input
acceleration of the UAV $j$ in 3D Euclidean space, respectively.

\begin{remark}
The second-order integrator in Eq.~\eqref{UAV_kinetic} plays a crucial role in demonstrating the kinetics of the UAVs, which captures the acceleration and hence is more suitable for the DGVF controller to achieve the upper-level heterogeneous coordination in Problem~\ref{problem_2} later. Additionally, in numerous prior works on UAV coordination (see, e.g., \cite{dong2014time}), it is a well-accepted operation to approximate the UAV's dynamics as second-order dynamics. This simplification allows for more focused research on the coordination navigation of the CDUS.
\end{remark}

\subsection{Multi-Vehicle Path Following}
Inspired by the definition of $\mathcal P^{phy}$ in Eq.~\eqref{HD_path_phy} for a single vehicle, we suppose that the $i$-th prescribed path $\mathcal P_i^{phy}$ for USV $i, i\in \mathcal V_1,$ is, 
\begin{align}\label{desired_path_USV}
\mathcal P_i^{phy}=&\big\{\sigma_i:=[\sigma_{i,x} , \sigma_{i,y}]\t\in\mathbb{R}^2~|~\nonumber\\
                              &\phi_{i,x}(\sigma_i):=\sigma_{i,x}-f_{i,x}(\omega_{i})=0, \nonumber\\
                              &\phi_{i,y}(\sigma_i):=\sigma_{i,y}-f_{i,y}(\omega_{i})=0\big\}, i\in\mathcal V_1,
\end{align}
where $\sigma_i$ are the coordinates, $\phi_{i,x}(\sigma_i), \phi_{i,y}(\sigma_i)$ denote the zero-level implicit functions, $f_{i,x}(\omega_{i}), f_{i,y}(\omega_{i})$ are the parametric functions, and $\omega_{i}$ represent the virtual coordinate of the $i$-th prescribed path $\mathcal P_i^{phy}$.

Let $\xi_i:=[\sigma_{i,x} , \sigma_{i,y}, \omega_{i}]\t\in\mathbb{R}^3$, it follows from Definition~\ref{definition_GVF} that the $i$-th prescribed path $\mathcal P_i^{phy}$ in \eqref{desired_path_USV} is transformed to a corresponding higher-dimensional path
\begin{align}
\label{desired_path_hgh_USV}
\mathcal P_i^{hgh}=&\big\{\xi_i\in\mathbb{R}^{3}~|~\nonumber\\
                              &\phi_{i,x}(\xi_i):=\sigma_{i,x}-f_{i,x}(\omega_{i})=0, \nonumber\\
                              &\phi_{i,y}(\xi_i):=\sigma_{i,y}-f_{i,y}(\omega_{i})=0\big\}, i\in\mathcal V_1.
\end{align}
Defining $\zeta_i:=[q_{i,x}, q_{i,y}, \omega_{i}]\t\in\mathbb{R}^{3}$ and substituting the position $q_i=[q_{i,x}, q_{i,y}]\t$ of the $i$-th USV in \eqref{USV_kinetic} into the $i$-th prescribed higher-dimensional path $\mathcal P_i^{hgh}$,  the path-following errors $\phi_{i,x}(\zeta_i), \phi_{i,y}(\zeta_i),$ between USV $i$ and $\mathcal P_i^{hgh}$ in \eqref{desired_path_hgh_USV} become
\begin{align}
\label{err_phi_USV}
\phi_{i,x}(\zeta_i)=&q_{i,x}-f_{i,x}(\omega_{i}),\nonumber\\
\phi_{i,y}(\zeta_i)=&q_{i,y}-f_{i,y}(\omega_{i}), i\in\mathcal V_1.
\end{align}
Analogously, we define  the $j$-th prescribed path $\mathcal P_j^{phy}$ for UAV $j, j\in\mathcal V_2$ in the UAV subgroup by
\begin{align}\label{desired_path_UAV}
\mathcal P_j^{phy}=&\big\{\sigma_j:=[\sigma_{j,x} , \sigma_{j,y},  \sigma_{j,z}]\t\in\mathbb{R}^3~|~\nonumber\\
                              & \phi_{j,x}(\sigma_j):=\sigma_{j,x}-f_{j,x}(\omega_{j})=0, \nonumber\\
                              & \phi_{j,y}(\sigma_j):=\sigma_{j,y}-f_{j,y}(\omega_{j})=0, \nonumber\\
                              & \phi_{j,z}(\sigma_j):=\sigma_{j,z}-f_{j,z}(\omega_{j})=0\big\}, j\in\mathcal V_2,
\end{align}
with the coordinates $\sigma_j$, the implicit functions $\phi_{j,x}, \phi_{j,y}, \phi_{j,z}$, and the parametric functions $f_{j,x}, f_{j,y}, f_{j,z}$. It follows from Eq.~\eqref{desired_path_UAV} that the higher-dimensional path $\mathcal P_j^{hgh}$ formulates
\begin{align}
\label{desired_path_hgh_UAV}
\mathcal P_j^{hgh}=&\big\{\xi_j\in\mathbb{R}^{4}~|~\nonumber\\
                              &\phi_{j,x}(\xi_j):=\sigma_{j,x}-f_{j,x}(\omega_{j})=0, \nonumber\\
                              &\phi_{j,y}(\xi_j):=\sigma_{j,y}-f_{j,y}(\omega_{j})=0, \nonumber\\ 
                              &\phi_{j,z}(\xi_j):=\sigma_{j,z}-f_{j,z}(\omega_{j})=0\big\}, j\in\mathcal V_2,
\end{align}
with $\xi_j:=[\sigma_{j,x}, \sigma_{j,y}, \sigma_{j,z}, \omega_{j}]\t\in\mathbb{R}^4$. Then, substituting the position $q_j=[q_{j,x}, q_{j,y}, q_{j,z}]\t$ of the $j$-th UAV in Eq.~\eqref{UAV_kinetic} into
Eq.~\eqref{desired_path_hgh_UAV} yields the path-following errors for UAV $j$,
\begin{align}
\label{err_phi_UAV}
\phi_{j,x}(\zeta_j)=&q_{j,x}-f_{j,x}(\omega_{j}),\nonumber\\
\phi_{j,y}(\zeta_j)=&q_{j,y}-f_{j,y}(\omega_{j}), \nonumber\\
\phi_{j,z}(\zeta_j)=&q_{j,z}-f_{j,z}(\omega_{j}), j\in\mathcal V_2,
\end{align}
with $\zeta_j:=[q_{j,x}, q_{j,y}, q_{j,z}, \omega_{j}]\t\in\mathbb{R}^{4}$.

Let $\Phi_i^{[1]}(\zeta_i):=[\phi_{i,x}(\zeta_i), \phi_{i,y}(\zeta_j)]\t,  i\in\mathcal V_1$ and $\Phi_j^{[2]}(\zeta_j):=[\phi_{j,x}(\zeta_j),\phi_{j,y}(\zeta_j), \phi_{j,z}(\zeta_j)]\t, j\in\mathcal V_2$, it then follows from 
Eqs.~\eqref{err_phi_USV} and \eqref{err_phi_UAV} that the objective of the multi-UAV-USV navigation becomes guiding the CDUS to converge to and maneuver along the prescribed paths, i.e.,
\begin{align}
\label{err_USV_UAV}
&\lim_{t\rightarrow\infty}\Phi_i^{[1]}(\zeta_i(t))=\mathbf{0}_2, \lim_{t\rightarrow\infty}\dot{\omega}_i(t)=\dot{\omega}_i^{\ast}\neq0, i\in\mathcal V_1,\nonumber\\ 
&\lim_{t\rightarrow\infty}\Phi_j^{[2]}(\zeta_j(t))=\mathbf{0}_3, \lim_{t\rightarrow\infty}\dot{\omega}_j(t)=\dot{\omega}_j^{\ast}\neq0, j\in\mathcal V_2,
\end{align}
where $\dot{\omega}_i^{\ast}$ and $\dot{\omega}_j^{\ast}$ are the desired derivatives of the virtual coordinates given in Eqs.~\eqref{derivative_omega_USV} and \eqref{derivative_omega_UAV} later.
\begin{remark}
The path-following errors $\Phi_i^{[1]}(\zeta_i(t)), i\in\mathcal V_1$  and $\Phi_j^{[2]}(\zeta_j(t)), j\in\mathcal V_2$ in Eq.~\eqref{err_USV_UAV} are of different dimensions. The former $\Phi_i^{[1]}(\zeta_j) \in\mathbb{R}^2$ for the USV subgroup runs in the 2D plane, whereas the latter $\Phi_j^{[2]}(\zeta_j)\in\mathbb{R}^3$ for the UAV subgroup is in the 3D Euclidean space. It inevitably poses challenging issues to the coordination of such heterogeneous CDUS later. The condition $\lim_{t\rightarrow\infty}\dot{\omega}_i(t)=\dot{\omega}_i^{\ast}\neq0, i\in\mathcal V_1\cup\mathcal V_2,$ in \eqref{err_USV_UAV} implicitly describes the path maneuvering of the CDUS guided by the high-dimensional GVF in \eqref{high_eq_GVF1}.
\end{remark}

\begin{figure}[!htb]
\centering
\includegraphics[width=\hsize]{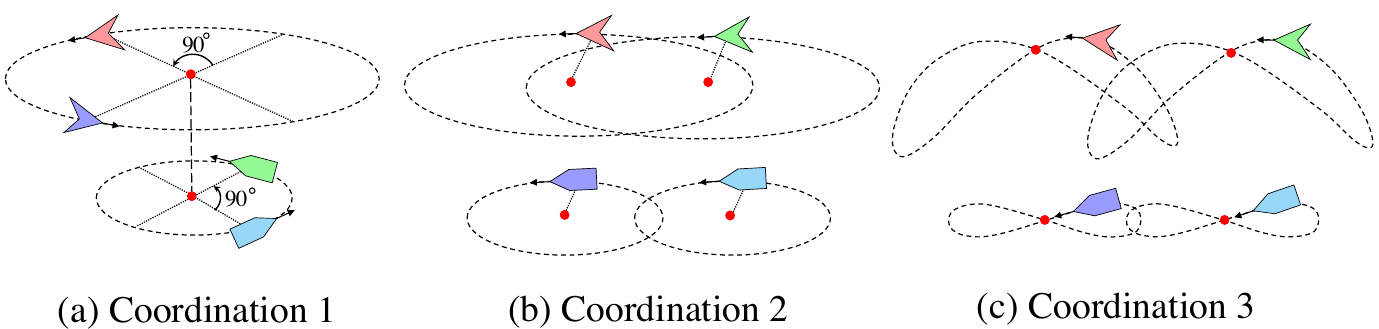}
\caption{Illustration of three distinct kinds of cross-domain coordinated navigation of the CDUS consisting of 2 UAVs and 2 USVs.
(a)~Concentric-circular-path coordinated navigation of the CDUS maneuvers with evenly distributed phases ($90^{\circ}$), where the two UAVs maneuver along a bigger prescribed circular path and two USVs maneuver a smaller one.
(b) Circular-path coordinated navigation of the CDUS maneuvers along the prescribed circular paths with the same radius but distinct circular centers, where the USVs and UAVs move with the same angle phases on the respective circle. (c) Self-intersecting-path coordinated navigation with two UAVs moving along the prescribed 3D Lissajous paths, and two USVs
moving along the 2D Lissajous paths.
 (The triangles are the UAVs and the vessel-like shapes are the USVs. The red points denote the circular center or self-intersected points of the prescribed paths.)}
\label{illustration_CDUS}
\end{figure}

\subsection{Cross-Domain Coordination Navigation}
\label{Preliminary_D}
Based on the path-following errors in Eq.~\eqref{err_USV_UAV} with the additional coordinate $\omega_i$, we are ready to introduce the cross-domain heterogeneous coordination of the CDUS.

Firstly, we define the communication topology of the heterogeneous CDUS by $\mathcal G (\mathcal V, \mathcal E)$, where $\mathcal V=\mathcal V_1\cup\mathcal V_2$ represents the node set of USVs and UAVs, and $\mathcal E:=\{(i,j), j\in\mathcal N_i, i\in\mathcal V\}$ the edge set. Here, $\mathcal N_i$ denotes the neighborhood set of vehicle~$i$. Moreover, let $A=[a_{i,j}]\in\mathbb{R}^{(n+m)\times (n+m)}$ be the adjacency matrix, where $a_{i,j}>0$ if $(i,j)\in\mathcal E$, otherwise $a_{i,j}=0$. Then, the Laplacian matrix of the CDUS $\mathcal L=[l_{i,j}]\t\in\mathbb{R}^{(n+m)\times(n+m)}$ is defined to be $l_{i,j}=-a_{i,j}, i\neq j$ and $l_{i,j}=\sum_{s=1}^{n+m}a_{i,s}, i=j$. Let $L_i\in\mathbb{R}^{1\times (n+m)}$ be the $i$-th row of the matrix $\mathcal L$.

 To achieve coordinated navigation of the CDUS along the prescribed paths $\mathcal P_i^{phy}, i\in\mathcal V$, we introduce the following reasonable assumptions.

\begin{assumption}
\label{assu_topology}
The communication topology of the CDUS $\mathcal G (\mathcal V, \mathcal E)$ is assumed to contain a connected directed spanning tree.
\end{assumption}

Since the motions of the CDUS are coordinated based on inter-vehicle parametric virtual coordinates $\omega_i, i\in\mathcal V$, Assumption~\ref{assu_topology} needs to ensure the full connectivity of the CDUS in a directed network (spanning tree) all along, which can be utilized to achieve $\lim_{t\rightarrow\infty}\omega_i(t)-\omega_j(t)=\Delta_{i,j}, j\in\mathcal N_i, i\in \mathcal V$ (i.e., Objective~3 of Definition~\ref{CPF_definition}) in Lemma~\ref{lemma_coordination_behavior} later. Therefore, Assumption~\ref{assu_topology} is a necessary condition for the heterogeneous coordination of the CDUS, which is also well-accepted in many previous MAS coordination works \cite{liu2021positive,chen2022event}.

\begin{assumption}
\label{assp_derivative}
The first and second derivatives of $f_{i,x}(\cdot), f_{i,y}(\cdot), i\in\mathcal V_1$ in Eq.~\eqref{desired_path_USV} and $f_{j,x}(\cdot), f_{j,y}(\cdot), f_{j,z}(\cdot), j\in\mathcal V_2$ in Eq.~\eqref{desired_path_UAV} are all bounded, i.e., 
\begin{align*}
&|f_{i,x}(\omega_{i})|\leq\eta_{i,x},~|f_{i,y}(\omega_{i})|\leq\eta_{i,y},\nonumber\\
&|f_{j,x}(\omega_{j})|\leq\eta_{j,x},~|f_{j,y}(\omega_{j})|\leq\eta_{j,y},~|f_{j,z}(\omega_{j})|\leq\eta_{j,z}, 
\end{align*}
with some unknown positive parameters $\eta_{i,x}\in\mathbb{R}^+, \eta_{i,y}\in\mathbb{R}^+, \eta_{j,x}\in\mathbb{R}^+, \eta_{j,y}\in\mathbb{R}^+,  \eta_{j,z}\in\mathbb{R}^+, i\in\mathcal V_1, j\in\mathcal V_2$.
\end{assumption}

Assumption~\ref{assp_derivative} illustrates that the prescribed path $\mathcal P_{i}^{phy}$ for the CDUS cannot change too quickly, which is reasonable in practice and necessary for rigorous convergence analysis later~\cite{yao2021singularity}.

Recalling the high-dimensional GVF $\chi^{hgh}$  in Definition~\ref{definition_GVF}, the path-following errors in Eq.~\eqref{err_USV_UAV}, the heterogenous neighborhood $\mathcal N_i$ and Assumptions~\ref{assp_error}-\ref{assp_derivative}, we can give the definition of the cross-domain coordinated navigation.

\begin{definition}
\label{CPF_definition}
(Cross-domain coordinated navigation) The group of heterogeneous CDUS ${\mathcal V~(\mathcal V_1 \cup \mathcal V_2})$ collectively forms a coordinated navigation whereas maneuvering along the prescribed path $\mathcal {P}_i^{phy}, i\in\mathcal{V}$ if the following objectives are satisfied,
\begin{enumerate}[1.]
\item 
 \textbf{(Path convergence)} All vehicles approach the prescribed paths $\mathcal P_i^{phy}$ asymptotically as $t\rightarrow\infty$, i.e.,  
\begin{align*}
\lim_{t\rightarrow\infty}\Phi_i^{[1]}(\zeta_i(t))=\mathbf{0}_2, i\in\mathcal V_1, \lim_{t\rightarrow\infty}\Phi_j^{[2]}(\zeta_j(t))=\mathbf{0}_3, j\in\mathcal V_2,
\end{align*}
with $\Phi_i^{[1]}(\zeta_i), i\in\mathcal V_1, \Phi_j^{[2]}(\zeta_j), j\in\mathcal V_2$ given in \eqref{err_USV_UAV}.

\item \textbf{(Path maneuvering)} All vehicles maneuver along the prescribed paths $\mathcal P_i^{phy}$ as $t\rightarrow\infty$, i.e.,  
\begin{align*}
\lim_{t\rightarrow\infty}\dot{\omega}_i(t)=\dot{\omega}_i^{\ast}\neq0, i\in\mathcal V_1\cup\mathcal V_2,
\end{align*}
with the virtual coordinate $\omega_i$ defined in Definition~\ref{definition_GVF} and $\dot{\omega}_i^{\ast}$ given in Eqs.~\eqref{derivative_omega_USV} and \eqref{derivative_omega_UAV}.

\item \textbf{(Heterogeneous coordination)} All vehicles coordinate the heterogeneous formation via inter-vehicle parametric virtual coordinates, i.e.,
\begin{align}
\label{specified_parametric_displace}
\lim_{t\rightarrow\infty} \omega_i(t)-\omega_j(t)=\Delta_{i,j}, j\in\mathcal N_i, i\in \mathcal V,
\end{align}
\end{enumerate}
where the constant $\Delta_{i,j}\in\mathbb{R}$ denotes a specified parametric displacement among the neighboring vehicles, and $\mathcal N_i$ is the neighborhood set of vehicle $i$.
\end{definition}

In Definition~\ref{CPF_definition}, Objectives 1 and 2 ensure the path following of the CDUS. Significantly, different from interacting the sophisticated Euclidean positions ($q_i, q_j$) in traditional heterogeneous coordination works~\cite{ong2021consensus,chen2019feedforward}, Objective 3 determines the heterogeneous coordination by only communicating a scalar, i.e., the virtual coordinate $\omega_{i}$, which substantially decreases the communication, sensor and calculation costs, and implicitly avoids the well-accepted assumption of homogeneous kernel~\cite{chen2019feedforward} in most of heterogeneous coordination problems.

\begin{remark}
By coordinating each vehicle's virtual coordinate $\omega_i$ among neighbor vehicles $\mathcal N_i$ in Objective 3 of Definition~\ref{CPF_definition}, the coordination of the CDUS swarm is not limited to achieving convergence to a single point for all vehicles anymore. As shown in Definition~\ref{CPF_definition}, it can also be employed as a path-following coordinated deployment of many different unmanned vehicles on whole prescribed and priori trajectories, which thus enriches applicable scenarios of the CDUS swarm. Then, by using different kinds of prescribed paths in Eqs.~\eqref{desired_path_USV} and \eqref{desired_path_UAV}, and distinct parametric displacement $\Delta_{i,j}$ in~Eq.~\eqref{specified_parametric_displace},  Definition~\ref{CPF_definition} illustrates various cross-domain heterogeneous navigation applications. As illustrated in Fig.~\ref{illustration_CDUS}, the CDUS performs three different kinds of coordinated navigation operations with different prescribed paths.
\end{remark}

\begin{remark}
The homogeneous kernel commonly encountered in the heterogeneous MAS problem, refers to a common dynamics that all agents can reach once synchronization is achieved~\cite{chen2019feedforward}. Such a “homogeneous kernel” does not always exists for a heterogeneous network as all agents contain different dynamics, which poses challenges in the heterogeneous coordination problems. However, in this paper, the heterogeneous coordination of the CDUS is achieved by merely communicating their virtual coordinates, which thus can implicitly avoid such challenges.
\end{remark}

\subsection{Problem Formulation}
Let $u_{i,r}, v_{i,r}\in\mathbb{R}$ be the reference guidance signals of the surge and sway velocities $u_i, v_i$ of USV $i, i\in\mathcal V_1$, respectively, then we can define the errors $\widetilde{u}_i, \widetilde{v}_i$ to be
\begin{align}
\label{guidance_err}
\widetilde{u}_i:=u_{i}-u_{i,r},~\widetilde{v}_i:=v_{i}-v_{i,r}.
\end{align}
It follows from Eqs.~\eqref{USV_kinetic} and \eqref{guidance_err} that  the kinematics of USV~$i, i\in\mathcal V_1$ can be rewritten in a compact form, 
\begin{align}
\label{USV_kinetic1}
\begin{bmatrix}
 \dot{q}_{i,x}\\
 \dot{q}_{i,y}
\end{bmatrix}
=&
\begin{bmatrix}
\cos\psi_i & -\sin\psi_i\\
\sin\psi_i & \cos\psi_i\\
\end{bmatrix}
\begin{bmatrix}
u_{i,r}+\widetilde{u}_i\\
v_{i,r}+\widetilde{v}_i\\
\end{bmatrix}\nonumber\\
=&\begin{bmatrix}
\cos\psi_i & -\sin\psi_i\\
\sin\psi_i & \cos\psi_i\\
\end{bmatrix}
\begin{bmatrix}
u_{i,r}\\
v_{i,r}\\
\end{bmatrix}
+\begin{bmatrix}
\widetilde{u}_{i,d}\\
\widetilde{v}_{i,d}
\end{bmatrix}
\end{align}
with 
\begin{align}
\label{body_guidance_err}
\widetilde{u}_{i,d}=&\widetilde{u}_i\cos\psi_i -\widetilde{v}_i\sin\psi_i,\nonumber\\
\widetilde{v}_{i,d}=&\widetilde{u}_i\cos\psi_i +\widetilde{v}_i\sin\psi_i.
\end{align}
Let $u_i^{\omega}$ be the derivative of the $i$-th virtual coordinate $\omega_i$, one has  
\begin{align}
\label{dynamic_omega}
\dot{\omega}_i=u_i^{\omega}, i\in\mathcal V_1.
\end{align}
Recalling the path-following errors $\phi_{i,x}(\zeta_i), \phi_{i,y}(\zeta_i)$ of USV~$i$  in Eq.~\eqref{err_phi_USV}, one has that the corresponding derivatives become
\begin{align}
\label{USV_gradient}
\dot{\phi}_{i,x}(\zeta_i)=&\nabla\phi_{i,x}(\zeta_i)\t\dot{\zeta}_i, \nonumber\\
\dot{\phi}_{i,y}(\zeta_i)=&\nabla\phi_{i,y}(\zeta_i)\t\dot{\zeta}_i, i\in\mathcal V_1,
\end{align}
where $\zeta_i=[q_{i,x}, q_{i,y}, \omega_{i}]\t$ and $\nabla\phi_{i,x}(\zeta_i), \nabla\phi_{i,y}(\zeta_i)$ represent the gradient w.r.t. the vector $\zeta_i$, which are calculated below 
\begin{align}
\label{gradient_fi}
\nabla\phi_{i,x}(\zeta_i):=&[1,0, -\partial{f}_{i,x}(\omega_i)]\t,\nonumber\\
\nabla\phi_{i,y}(\zeta_i):=&[0,1, -\partial{f}_{i,y}(\omega_i)]\t, i\in\mathcal V_1,
\end{align}
where $\partial{f}_{i, x}(\omega_i), \partial{f}_{i, y}(\omega_i)$ are the partial derivatives of $f_{i,x}(\omega_i)$ $f_{i,y}(\omega_i)$ w.r.t. $\omega_{i}$
\begin{align}
\label{gradient_f_omega12}
\partial{f}_{i,x}(\omega_i):=\frac{\partial f_{i,x}(\omega_i)}{\partial\omega_{i}},~\partial{f}_{i,y}(\omega_i):=\frac{\partial f_{i,y}(\omega_{i})}{\partial\omega_{i}}.
\end{align}
Substituting Eqs.~\eqref{USV_kinetic1} and~\eqref{gradient_fi} into Eq.~\eqref{USV_gradient}, and rewriting Eqs.~\eqref{dynamic_omega},~\eqref{USV_gradient} into a compact form, one has that the closed-loop system of USV $i, i\in\mathcal V_1$, below
\begin{align}
\label{err_dynamic_USV1}
\begin{bmatrix}
\dot{\phi}_{i,x}(\zeta_i)\\
\dot{\phi}_{i,y}(\zeta_i)\\
\dot{\omega}_{i}\\
\end{bmatrix}
=
D_i 
\begin{bmatrix}
u_{i,r}\\
v_{i,r}\\
u_{i}^{\omega}\\
\end{bmatrix}
+
\begin{bmatrix}
\widetilde{u}_{i,d}\\
\widetilde{v}_{i,d}\\
0
\end{bmatrix}
\end{align}
with
\begin{align*}
D_i=& \begin{bmatrix}                      
                       \cos\psi_i & -\sin\psi_i &-\partial{f}_{i,x}(\omega_i)\\
                       \sin\psi_i & \cos\psi_i & -\partial{f}_{i,y}(\omega_i)\\
                       0 & 0  & 1\\
                       \end{bmatrix}\in \mathbb{R}^{3\times 3}.
\end{align*}
Analogously, let $p_{j,x}^r, p_{j,y}^r, p_{j,z}^r$ be the desired guidance signals of the velocities of UAV $j, j\in\mathcal V_2$, respectively, then we define $\widetilde{p}_{j, x}, \widetilde{p}_{j, y}, \widetilde{p}_{j, z}$ as the errors between the desired signals and the actual velocities below, 
\begin{align}
\label{guidance_err_UAV}
\widetilde{p}_{j, x}:=&p_{j,x}-p_{j,x}^r,\nonumber\\
\widetilde{p}_{j, y}:=&p_{j,y}-p_{j,y}^r,\nonumber\\
\widetilde{p}_{j, z}:=&p_{j,z}-p_{j,z}^r, j\in\mathcal V_2.
\end{align}
Substituting Eq.~\eqref{guidance_err_UAV} into the kinematics of UAV $j$ in Eq.~\eqref{UAV_kinetic} yields 
\begin{align}
\label{UAV_kinetic_1}
 \dot{q}_{j,x} &=p_{j,x}^r+\widetilde{p}_{j, x}, \nonumber\\
 \dot{q}_{j,y} &=p_{j,y}^r+\widetilde{p}_{j, y}, \nonumber\\
 \dot{q}_{j,z} &=p_{j,z}^r+\widetilde{p}_{j, z}.
\end{align}
Let $\zeta_j:=[q_{j,x}, q_{j,y}, q_{j,z}, \omega_{j}]\t\in \mathbb{R}^4$ for UAV $j, j\in\mathcal V_2$, and we can calculate the gradient of $\phi_{j,x}(\zeta_j), \phi_{j,y}(\zeta_j), \phi_{j,z}(\zeta_j)$ in Eq.~\eqref{err_phi_UAV} along the vector $\zeta_j$ 
\begin{align}
\label{gradient_fi_UAV}
\nabla\phi_{j,x}(\zeta_j):=&[1,0,0, -\partial{f}_{j,x}(\omega_j)]\t,\nonumber\\
\nabla\phi_{j,y}(\zeta_j):=&[0,1,0, -\partial{f}_{j,y}(\omega_j)]\t,\nonumber\\
\nabla\phi_{j,z}(\zeta_j):=&[0,0,1, -\partial{f}_{j,z}(\omega_j)]\t, j\in\mathcal V_2,
\end{align}
with
\begin{align}
\label{gradient_f_omega345}
&\partial{f}_{j,x}(\omega_j):=\frac{\partial f_{j,x}(\omega_j)}{\partial\omega_{j}},~\partial{f}_{j,y}(\omega_j):=\frac{\partial f_{j,y}(\omega_{j})}{\partial\omega_{j}},\nonumber\\
&\partial{f}_{j,z}(\omega_j):=\frac{\partial f_{j,z}(\omega_j)}{\partial\omega_{j}}.
\end{align}
It follows from Eqs.~\eqref{dynamic_omega}, \eqref{UAV_kinetic_1}, \eqref{gradient_fi_UAV} that the closed-loop system $\{\phi_{j,x}(\zeta_j), \phi_{j,y}(\zeta_j), \phi_{j,z}(\zeta_j), \omega_j\}$
 of UAV $j, j\in\mathcal V_2$ is formulated below
\begin{align}
\label{err_dynamic_UAV1}
\begin{bmatrix}
\dot{\phi}_{j,x}(\zeta_j)\\
\dot{\phi}_{j,y}(\zeta_j)\\
\dot{\phi}_{j,z}(\zeta_j)\\
\dot{\omega}_{j}\\
\end{bmatrix}
=
D_j
\begin{bmatrix}
p_{j,x}^r\\
p_{j,y}^r\\
p_{j,z}^r\\
u_{j}^{\omega}\\
\end{bmatrix}
+
\begin{bmatrix}
\widetilde{p}_{j, x}\\
\widetilde{p}_{j, y}\\
\widetilde{p}_{j, z}\\
0
\end{bmatrix}
\end{align}
with 
\begin{align*}
D_j=& \begin{bmatrix}                      
                       1 & 0 & 0 &-\partial{f}_{j,x}(\omega_j)\\
                       0 & 1 & 0 &-\partial{f}_{j,y}(\omega_j)\\
                       0 & 0  & 1 & -\partial{f}_{j,z}(\omega_j)\\
                       0 & 0 & 0 & 1\\
                       \end{bmatrix}\in \mathbb{R}^{4\times 4}.
\end{align*}
In the presence of heterogeneous USVs' and UAVs' dynamics in Eqs.~\eqref{USV_kinetic} and \eqref{UAV_kinetic}, the closed-loop systems \eqref{err_dynamic_USV1}  and \eqref{err_dynamic_UAV1}  are established by a hierarchical control architecture in terms of upper-level heterogeneous guidance control  and lower-level signal tracking regulation, which then implies that the coordinated navigation of the CDUS are decoupled into two problems.

\begin{figure}[!htb]
\centering
\includegraphics[width=\hsize]{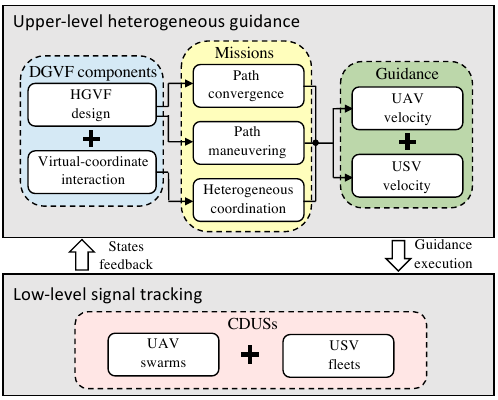}
\caption{A hierarchical control architecture, which consists of an upper-level  heterogeneous guidance velocity controller and a lower-level signal tracking regulator.}
\label{control_structure}
\end{figure}

\begin{problem}
\label{problem_1}
 (Lower-level signal tracking regulation) Design a regulated control law $\{\tau_{i,u}, \tau_{i,r}\}, i\in\mathcal V_1,$ and $\{\tau_{j, x}, \tau_{j, y}, \tau_{j, z}\}, j\in\mathcal V_2,$ for the dynamics of USV $i$ in \eqref{USV_kinematic} and UAV $j$ in \eqref{UAV_kinetic} such that the desired guidance velocities of the CDUS $u_{i}^r, v_{i}^r, p_{j,x}^r, p_{j,y}^r, p_{j,z}^r, i\in \mathcal V_1, j\in\mathcal V_2 $,
can be exponentially tracked, i.e., 
\begin{align}
\label{conver_condi_1}
&\lim_{t\rightarrow\infty}\widetilde{u}_i(t)=0,
\lim_{t\rightarrow\infty}\widetilde{v}_i(t)=0,\nonumber\\
&\lim_{t\rightarrow\infty}\widetilde{p}_{j, x}(t)=0, 
\lim_{t\rightarrow\infty}\widetilde{p}_{j, y}(t)=0, 
\lim_{t\rightarrow\infty}\widetilde{p}_{j, z}(t)=0,
\end{align}
exponentially.
\end{problem}

\begin{problem}
\label{problem_2}
(Upper-level heterogeneous guidance control):  Design a  DGVF guidance protocol
\begin{align}
\label{pro_desired_signal}
[u_{i}^r, v_{i}^r, u_{i}^{\omega}]\t:=& \chi_i^{hgh}(\psi_i, \phi_{i,x}, \phi_{i,y}, \partial{f}_{i,x}, \partial{f}_{i,y}, \omega_{i}, \omega_{k}), \nonumber\\
                                                                  & \forall i\in{\cal V}_1, k\in \mathcal N_i,\nonumber\\
[p_{j,x}^r, p_{j,y}^r, p_{j,z}^r, u_{j}^{\omega} ]\t:=& \chi_j^{hgh}(\phi_{j,x},\phi_{j,y}, \phi_{j,z}, \partial{f}_{j,x}, \partial{f}_{j,y}, \partial{f}_{j,z},  \nonumber\\
                                                                  & \omega_{j}, \omega_{k}), \forall j\in{\cal V}_2, k\in \mathcal N_j,
\end{align} 
for the CDUS composed of~\eqref{err_dynamic_USV1} and~\eqref{err_dynamic_UAV1} to fulfill heterogeneous navigation (i.e., the three objectives in Definition~\ref{CPF_definition}) if Problem~\ref{problem_1} is solved. 
Here, $\phi_{i,x}, \phi_{i,y}, \partial{f}_{i,x}, \partial{f}_{i,y}$ denote $\phi_{i,x}(\zeta_i), \phi_{i,y}(\zeta_i), $ $ \partial{f}_{i,x}(\omega_i)/\partial \omega_i$, $\partial{f}_{i,y}(\omega_i)/\partial \omega_i$ of USV $i$, and $\phi_{j,x},\phi_{j,y}, \phi_{j,z},  \partial{f}_{j,x}, $ $\partial{f}_{j,y}$, $\partial{f}_{j,z}$ represent $\phi_{j,x}(\zeta_j),\phi_{j,y}(\zeta_j), \phi_{j,z}(\zeta_j),  \partial{f}_{j,x}(\omega_j)/\partial \omega_j, \partial{f}_{j,y}(\omega_j)/\partial \omega_j,$ $ \partial{f}_{j,z}(\omega_j)/\partial \omega_j$ of UAV $j$ for conciseness if there exists no ambiguity.
\end{problem}

\begin{remark}
Illustration of such a hierarchical control architecture is given in Fig.~\ref{control_structure}. In particular, Problem~\ref{problem_1} ensures the fundamental condition for the rigorous convergence of the upper-level heterogeneous coordination in Problem~\ref{problem_2}. Moreover, by regarding the low-level tracking errors as
the external disturbances, Problem~\ref{problem_2} is transferred to the heterogeneous coordination subject to exponentially vanishing tracking errors from Problem~\ref{problem_1}.
\end{remark}

\section{Main Technical Result }

First of all, recalling the dynamics of USVs and UAVs in Eqs.~\eqref{USV_kinematic} and \eqref{UAV_kinetic}, the corresponding control laws $\tau_{i,u}, \tau_{i,r} $ for USV $i, i\in \mathcal V_1$, in Problem~\ref{problem_1} are designed according to \cite{hu2021bearing}
\begin{align}
\label{regulated_law_USV}
\tau_{i,u}=&\frac{1}{\epsilon_3u_{i}^r}\big(\dot{u}_{i}^ru_i-\epsilon_1u_{i}^ru_i-\epsilon_2u_i^rv_ir_i-b_{i,1}u_{i}\widetilde{u}_i\big),\nonumber\\
\tau_{i.r}=&\frac{1}{\epsilon_5}\big(-\epsilon_4r_i-b_{i,2}\widetilde{r}_i+\dot{r}_i^r\big), i\in\mathcal V_1,
\end{align}
with $b_{i,1}, b_{i,2}\in\mathbb{R}^+$ being the positive gains, $\epsilon_i, i\in\mathbb{Z}_1^5, u_i, v_i,$  $r_i$  given in \eqref{USV_kinematic}, and $u_{i,r}, \widetilde{u}_i$ defined in \eqref{guidance_err}. Here, $\widetilde{r}_i$ and $r_i^r$ denote the yaw velocity error and the desired yaw velocity, respectively, which are formulated below,
\begin{align}
\label{regulated_law_USV2}
\widetilde{r}_i=&r_i-\widetilde{r}_i^r, \nonumber\\
r_i^r=&\frac{1}{\epsilon_7u_i}(-\epsilon_6v_i+\dot{v}_i^r-b_{i,3}\widetilde{v}_i),
\end{align}
where $b_{i,3}\in\mathbb{R}^+$ is the positive gain, $\epsilon_6, \epsilon_7$ are the parameters in Eq.~\eqref{USV_kinematic},  and $v_i^r, \widetilde{v}_i$ are given in~\eqref{guidance_err}. As for UAV $j, j\in\mathcal V_2$, the regulated control laws are designed below
\begin{align}
\label{regulated_law_UAV}
\tau_{j, x}=& -b_{i,4}\widetilde{p}_{j, x}+\dot{p}_{j,x}^r,\nonumber\\
\tau_{j, y}=& -b_{i,4}\widetilde{p}_{j, y}+\dot{p}_{j,y}^r,\nonumber\\
\tau_{j, z}=& -b_{i,4}\widetilde{p}_{j, z}+\dot{p}_{j,z}^r,
\end{align}
where $b_{i,4}\in\mathbb{R}^+$ is the positive gain and $p_{j,x}^r, p_{j,y}^r, p_{j,z}^r$, $\widetilde{p}_{j, x}, \widetilde{p}_{j, y}, \widetilde{p}_{j, z}$ are given in \eqref{guidance_err_UAV}.

With the assistance of the detailed controller in Eqs.~\eqref{regulated_law_USV}, \eqref{regulated_law_USV2} and \eqref{regulated_law_UAV}, the lower-level signal tracking in Problem~\ref{problem_1} can be addressed in the following lemma.

\begin{lemma}
\label{lemma_signal_regulation}
For the desired guidance signals $u_{i}^r, v_{i}^r, p_{j,x}^r,$ $p_{j,y}^r, p_{j,z}^r, i\in \mathcal V_1, j\in\mathcal V_2$, the dynamics of USV $i, i\in\mathcal V_1,$ governed by~\eqref{USV_kinematic}, \eqref{regulated_law_USV} and \eqref{regulated_law_USV2}, and UAV $j, j\in\mathcal V_2,$ governed by  \eqref{UAV_kinetic} and \eqref{regulated_law_UAV} are able to achieve the lower-level signal tracking regulation, i.e., the condition~\eqref{conver_condi_1} in Problem~\ref{problem_1}.
\end{lemma}

\begin{proof}
On one hand, it follows from Eqs.~\eqref{USV_kinematic}, \eqref{regulated_law_USV} and \eqref{regulated_law_USV2} that the conditions of $\lim_{t\rightarrow\infty}\widetilde{u}_i(t)=0,
\lim_{t\rightarrow\infty}\widetilde{v}_i(t)=0$ for USV $i, i\in\mathcal V_1$ are fulfilled, whose proof is similar to \cite{hu2021bearing} and thus omitted.

On the other hand, taking derivative of $\widetilde{p}_{j,x}, \widetilde{p}_{j,y}, \widetilde{p}_{j,z}$ and substituting Eq.~\eqref{regulated_law_UAV} into Eq.~\eqref{UAV_kinetic} yields 
\begin{align}
\label{err_UAV_signal}
\dot{\widetilde{p}}_{j,x}=& -b_{i,4}\widetilde{p}_{j, x},\nonumber\\
\dot{\widetilde{p}}_{j,y}=& -b_{i,4}\widetilde{p}_{j, y},\nonumber\\
\dot{\widetilde{p}}_{j,z}=& -b_{i,4}\widetilde{p}_{j, z}.
\end{align}
It is straightforward to show from \eqref{err_UAV_signal} that
$$\lim_{t\rightarrow\infty}\widetilde{p}_{j, x}(t)=0, 
\lim_{t\rightarrow\infty}\widetilde{p}_{j, y}(t)=0, 
\lim_{t\rightarrow\infty}\widetilde{p}_{j, z}(t)=0,$$
converging exponentially. The proof is thus completed.
\end{proof}
\begin{remark}
In Lemma~\ref{lemma_signal_regulation}, the errors $\widetilde{u}_i, \widetilde{v}_i, \widetilde{p}_{j,x}, \widetilde{p}_{j,y}, \widetilde{p}_{j,z}, i\in\mathcal V_1, j\in\mathcal V_2$, albeit convergent to zeros exponentially, influence the stability of the upper-level heterogeneous coordination. It may exhibit divergence for sufficiently small errors induced by the lower-level signal tracking in Problem~\ref{problem_1}~\cite{freeman1995global}, i.e., 
the error states $\Phi_i^{[1]}(\zeta_i(t)), \Phi_j^{[2]}(\zeta_j(t)), \omega_k, i\in\mathcal V_1, j\in\mathcal V_2, k\in\mathcal V,$ in Eq.~\eqref{err_USV_UAV} satisfy 
\begin{align*}
&\lim_{t\rightarrow T_1}\Phi_i^{[1]}(\zeta_i(T_1))=\infty, i\in\mathcal V_1,~or~\nonumber\\
&\lim_{t\rightarrow T_1}\Phi_j^{[2]}(\zeta_j(T_1))=\infty, j\in\mathcal V_2,~or~\nonumber\\
&\lim_{t\rightarrow T_1}\omega_k(t)=\infty, k\in\mathcal V,
\end{align*}
for a positive time $T_1>0$, which will be prevented in Lemma~\ref{finite_time_behavior} shown later.
\end{remark}

Inspired by the high-dimensional GVF $\chi^{hgh}$ in \eqref{high_eq_GVF1}, it follows from Eqs.~\eqref{err_phi_USV}, \eqref{gradient_fi} and \eqref{err_dynamic_USV1} that the desired guidance signals $u_{i, r}, v_{i, r}, u_{i}^{\omega}$ for USV $i, i\in\mathcal V_1$ are designed below,  
\begin{align}
\label{control_law_USV}
u_{i, r}:=&\big((-1)^3 \partial{f}_{i,x}-k_{i,x}\phi_{i,x}\big)\cos\psi_i\nonumber\\
              &+\big((-1)^3 \partial{f}_{i,y}-k_{i,y}\phi_{i,y}\big)\sin\psi_i,\nonumber\\
v_{i, r}:=&-\big((-1)^3 \partial{f}_{i,x}-k_{i,x}\phi_{i,x}\big)\sin\psi_i\nonumber\\
              &+\big((-1)^3 \partial{f}_{i,y}-k_{i,y}\phi_{i,y}\big)\cos\psi_i,\nonumber\\
u_{i}^{\omega}:=&(-1)^3+k_{i,x}\phi_{i,x}\partial{f}_{i,x}+k_{i,y}\phi_{i,y}\partial{f}_{i,y}\nonumber\\
                       &-c_i\sum_{k\in\mathcal N_i}(\omega_{i}-\omega_{k}-\Delta_{i,k}), i\in\mathcal V_1,                                      
\end{align}
where $\psi_i$ is the heading angle of USV $i$, $k_{i,x}, k_{i,y}, c_i \in\mathbb R^{+}, i\in\mathcal V_1,$ the corresponding control gains, $\phi_{i,x}, \phi_{i,y},$ $ \partial f_{i,x}, \partial f_{i,y}$ $\omega_i$, given in \eqref{pro_desired_signal}, $\mathcal N_i$ the neighborhood set of vehicle $i$, and $\Delta_{i,k}$ the specified parametric displacement among the neighboring vehicles.

Analogously, it follows from Eqs.~\eqref{err_phi_UAV}, \eqref{gradient_fi_UAV} and \eqref{err_dynamic_UAV1} that the desired guidance signals for UAV $j, j\in\mathcal V_2$ satisfy 
\begin{align}
\label{control_law_UAV}
p_{j, x}^r:=&(-1)^3 \partial{f}_{j,x}-k_{j,x}\phi_{j,x},\nonumber\\
p_{j, y}^r:=&(-1)^3 \partial{f}_{j,y}-k_{j,y}\phi_{j,y},\nonumber\\
p_{j, z}^r:=&(-1)^3 \partial{f}_{j,z}-k_{j,z}\phi_{j,z},\nonumber\\
u_{j}^{\omega}:=& (-1)^3+k_{j,x}\phi_{j,x}\partial{f}_{j,x}+k_{j,y}\phi_{j,y}\partial{f}_{j,y}+k_{j,z}\phi_{j,z}\partial{f}_{j,z}\nonumber\\
                       &-c_j\sum_{k\in\mathcal N_j}(\omega_{j}-\omega_{k}-\Delta_{j,k}), j\in\mathcal V_2,
\end{align}
with the corresponding gains $k_{j,x}\in\mathbb{R}^+, k_{j,y}\in\mathbb{R}^+, k_{j,z}\in\mathbb{R}^+, c_j\in \mathbb{R}^{+}, j\in\mathcal V_2$, and $\phi_{j,x}, \phi_{j,y}, \phi_{j,z}, \partial f_{j,x}, \partial f_{j,y},  \partial f_{j,z}$ $\omega_i, \Delta_{i,k} $ given in Eq.~\eqref{pro_desired_signal}.

To illustrate the upper-level heterogeneous guidance in Problem~\ref{problem_2} more clearly, we first prevent the divergence behavior, and then prove path convergence and maneuvering, and heterogeneous coordination in three lemmas (i.e., Objectives 1-3 in Definition~\ref{CPF_definition}), respectively.

\begin{lemma}
\label{finite_time_behavior}
Under Lemma~\ref{lemma_signal_regulation}, a CDUS composed of \eqref{USV_kinetic}, \eqref{UAV_kinetic}, \eqref{control_law_USV}, \eqref{control_law_UAV} ensures that there exist no divergence behaviors during the navigation process, i.e., 
$$\|\Phi_i^{[1]}(t)\| \leq \delta_{1}, \|\Phi_j^{[2]}(t)\|\leq\delta_{2}, |\omega_k(t)|\leq\delta_3, \forall t>0,$$
where $\Phi_i^{[1]}, \Phi_j^{[2]}, \omega_k, i\in\mathcal V_1, j\in\mathcal V_2, k\in\mathcal V$ are given in Eq.~\eqref{err_USV_UAV} and $\delta_1\in\mathbb{R}^+, \delta_2\in\mathbb{R}^+, \delta_3\in\mathbb{R}^+$ unknown positive constants.
\end{lemma}

\begin{proof}
Suppose USV $i, i\in\mathcal V,$ guided by the GVF $\chi^{hgh}$ in Eq.~\eqref{high_eq_GVF1} already moves along a prescribed path in the 2D plane, we define the corresponding 
desired virtual coordinate of vehicle $i$ by $\omega_i^{\ast}$, whose derivative satisfies  
\begin{align}
\label{derivative_omega_USV}
\dot{\omega}_i^{\ast}=(-1)^3, i\in\mathcal V_1.
\end{align}
Here, the value of $\dot{\omega}_i^{\ast}$ in Eq.~\eqref{derivative_omega_USV} is consistent with $u_i^{\omega}$ in~\eqref{control_law_USV} when $\phi_{i,x}=0, \phi_{i,y}=0, \omega_{i}^{\ast}-\omega_{k}^{\ast}=\Delta_{i,k}, i\in\mathcal V_1, k\in\mathcal N_i$, which indicates that the maneuvering direction of USVs governed by~\eqref{control_law_USV} is the same as the UAVs.

Let $\widetilde{\omega}_{i}$ be the error between the virtual coordinate $\omega_i$ and the desired coordinate $\omega_{i}^*$ of USV $i$,
\begin{align}
\label{err_virc_USV}
\widetilde{\omega}_{i}:={\omega}_{i}-{\omega}_{i}^*, i\in \mathcal V_1.
\end{align}
Taking the derivative of $\widetilde{\omega}_{i}$ in Eq.~\eqref{err_virc_USV} and substituting the desired law $u_i^{\omega}$ in Eq.~\eqref{control_law_USV} and the dynamic $\dot{\omega}_i^*$ into \eqref{derivative_omega_USV} yields
\begin{align}
\label{error_virtual_coordinate}
\dot{\widetilde{\omega}}_{i}=&k_{i,x}\phi_{i,x}\partial{f}_{i,x}+k_{i,y}\phi_{i,y}\partial{f}_{i,y}\nonumber\\
			           &-c_i\sum_{k\in\mathcal N_i}(\widetilde{\omega}_{i}-\widetilde{\omega}_{k}), i\in\mathcal V_1.
\end{align}
Recalling $\Phi_i^{[1]}$ in \eqref{err_USV_UAV} and defining $F_i^{[1]}:=[\partial{f}_{i,x}, \partial{f}_{i,y}]\t\in\mathbb{R}^2, K_i^{[1]}=\mbox{diag}\{k_{i,x}, k_{i,y}\}\in\mathbb{R}^{2\times 2}, \varepsilon_i:=[\widetilde{u}_{i,d}, \widetilde{v}_{i,d}]\t\in\mathbb{R}^2, i\in\mathcal V_1$, it follows from Eqs.~\eqref{err_dynamic_USV1},~\eqref{control_law_USV},~\eqref{error_virtual_coordinate} that
\begin{align}
\label{err_dynamic_USV2}
\begin{bmatrix}
\dot{\Phi}_i^{[1]}\\
\dot{\widetilde{\omega}}_{i}\\
\end{bmatrix}
=& 
\begin{bmatrix}
-K_i^{[1]}\big(I_2+F_i^{[1]}(F_i^{[1]})\t\big)\Phi_i^{[1]}\\
(\Phi_i^{[1]})\t K_i^{[1]}F_i^{[1]}
\end{bmatrix}\nonumber\\
&+
\begin{bmatrix}
-F_i^{[1]}\big((-c_i\sum\limits_{k\in\mathcal N_i}(\widetilde{\omega}_{i}-\widetilde{\omega}_{k})\big)+\varepsilon_i\\
-c_i\sum\limits_{k\in\mathcal N_i}(\widetilde{\omega}_{i}-\widetilde{\omega}_{k})
\end{bmatrix}.
\end{align}
Analogously, for the UAV subgroup, we also define the corresponding error between the virtual coordinate $\omega_{j}$ and the desired coordinate $\omega_{j}^{\ast}$ as 
\begin{align}
\label{err_virc_UAV}
\widetilde{\omega}_{j}:={\omega}_{j}-{\omega}_{j}^*, j\in \mathcal V_2,
\end{align}
with the derivative of $\omega_j^{\ast}$ satisfying
\begin{align}
\label{derivative_omega_UAV}
\dot{\omega}_j^{\ast}=(-1)^3, j\in\mathcal V_2.
\end{align}
Then, it follows from Eqs.~\eqref{err_dynamic_UAV1}, \eqref{control_law_UAV}, \eqref{err_virc_UAV}, \eqref{derivative_omega_UAV} that the closed-loop system for the UAV $j, j\in\mathcal V_2$ is formulated below
\begin{align}
\label{err_dynamic_UAV2}
\begin{bmatrix}
\dot{\Phi}_j^{[2]}\\
\dot{\widetilde{\omega}}_{j}\\
\end{bmatrix}
=& 
\begin{bmatrix}
-K_j^{[2]}\big(I_3+F_j^{[2]}(F_j^{[2]})\t\big)\Phi_j^{[2]}\\
(\Phi_j^{[2]})\t K_j^{[2]}F_j^{[2]}
\end{bmatrix}\nonumber\\
&+
\begin{bmatrix}
-F_j^{[2]}\big((-c_j\sum_{k\in\mathcal N_j}(\widetilde{\omega}_{j}-\widetilde{\omega}_{k})\big)+\vartheta_j\\
-c_j\sum_{k\in\mathcal N_j}(\widetilde{\omega}_{j}-\widetilde{\omega}_{k})
\end{bmatrix}
\end{align}
with $\Phi_j^{[2]}:=[\phi_{j,x}, \phi_{j,y}, \phi_{j,z}]\t\in\mathbb{R}^3, F_j^{[2]}:=[\partial{f}_{j,x}, \partial{f}_{j,y},$ $ \partial{f}_{j,z}]\t\in\mathbb{R}^3, K_j^{[2]}=\mbox{diag}\{k_{j,x}, k_{j,y}, k_{j,z}\}\in\mathbb{R}^{3\times 3}$ and $\vartheta_j:=[\widetilde{p}_{j, x}, \widetilde{p}_{j, y}, \widetilde{p}_{j, z}]\t\in\mathbb{R}^3, j\in\mathcal V_2$.

Rewriting the states of USVs and UAVs in Eqs.~\eqref{err_dynamic_USV2} and~\eqref{err_dynamic_UAV2} in compact vectors below,
\begin{align}
\label{CDUS_vector}
\boldsymbol{\Phi}^{[1]}:=&[(\Phi_1^{[1]})\t, \dots, (\Phi_n^{[1]})\t ]\t\in\mathbb{R}^{2n}, \nonumber\\
\boldsymbol{\Phi}^{[2]}:=&[(\Phi_{n+1}^{[2]})\t,\dots, (\Phi_{n+m}^{[2]})\t ]\t\in\mathbb{R}^{3m}, \nonumber\\ 
\boldsymbol{K}^{[1]}:=&\mbox{diag}\{K_1^{[1]}, \dots, K_n^{[1]}\}\in\mathbb{R}^{2n\times 2n}, \nonumber\\ 
\boldsymbol{K}^{[2]}:=&\mbox{diag}\{K_{n+1}^{[2]}, \dots, K_{n+m}^{[2]}\}\in\mathbb{R}^{3m\times 3m}, \nonumber\\ 
\boldsymbol{F}^{[1]}:=&\mbox{diag}\{F_1^{[1]}, \dots, F_n^{[1]}\}\in\mathbb{R}^{2n\times 2n},\nonumber\\ 
\boldsymbol{F}^{[2]}:=&\mbox{diag}\{F_{n+1}^{[2]}, \dots, F_{n+m}^{[2]}\}\in\mathbb{R}^{3m\times 3m},\nonumber\\ 
\boldsymbol{\widetilde{\omega}}:=&[\widetilde{\omega}_{1}, \dots, \widetilde{\omega}_{n+m}]\t\in\mathbb{R}^{n+m},\nonumber\\ 
C:=&\mbox{diag}\{c_{1}, \dots, c_{n+m}\}\t\in\mathbb{R}^{(n+m)\times(n+m)},
\end{align}
we can pick a Lyapunov function candidate below,
\begin{align}
\label{func_V}
V(\boldsymbol{\Phi}^{[1]}, \boldsymbol{\Phi}^{[2]},\boldsymbol{\widetilde{\omega}})=&\frac{1}{2}\Big\{(\boldsymbol{\Phi}^{[1]})\t\boldsymbol{K}^{[1]}\boldsymbol{\Phi}^{[1]} \nonumber\\
															      &+ (\boldsymbol{\Phi}^{[2]})\t\boldsymbol{K}^{[2]}\boldsymbol{\Phi}^{[2]}+ \boldsymbol{\widetilde{\omega}}
															      	\t C \mathcal L\boldsymbol{\widetilde{\omega}}\Big\}, 
\end{align}
where $\mathcal L\in\mathbb{R}^{(n+m)\times(n+m)}$ is the connection matrix of the CDUS.
Taking the derivative of $V(\boldsymbol{\Phi}^{[1]}, \boldsymbol{\Phi}^{[2]},\boldsymbol{\widetilde{\omega}})$ in \eqref{func_V} yields
\begin{align}
\label{d_func_V}
\dot{V}=&(\boldsymbol{\Phi}^{[1]})\t\boldsymbol{K}^{[1]}\dot{\boldsymbol{\Phi}}^{[1]} + (\boldsymbol{\Phi}^{[2]})\t\boldsymbol{K}^{[2]}\dot{\boldsymbol{\Phi}}^{[2]}+ \boldsymbol{\widetilde{\omega}}
															      	\t C\mathcal L\dot{\boldsymbol{\widetilde{\omega}}}	.											
\end{align}
On one hand, using the definitions of $\boldsymbol{\Phi}^{[1]}, \boldsymbol{\Phi}^{[2]}$ in \eqref{CDUS_vector}, it then follows from the dynamics of $\boldsymbol{\Phi}^{[1]}, \boldsymbol{\Phi}^{[2]}$ in Eqs.~\eqref{err_dynamic_USV2},~\eqref{err_dynamic_UAV2} and \eqref{d_func_V} that
\begin{align}
\label{V_equation}
(\boldsymbol{\Phi}^{[1]})\t\boldsymbol{K}^{[1]}\dot{\boldsymbol{\Phi}}^{[1]}
														=&\sum_{i=1}^n\Big\{-(\Phi_i^{[1]})\t K_i^{[1]}K_i^{[1]}\Phi_i^{[1]}\nonumber\\
														&-(\Phi_i^{[1]})\t K_i^{[1]}F_i^{[1]}(F_i^{[1]})\t K_i^{[1]}\Phi_i^{[1]}\nonumber\\
														&+c_i(\Phi_i^{[1]})\t K_i^{[1]}F_i^{[1]}L_i\boldsymbol{\widetilde{\omega}}+(\Phi_i^{[1]})\t K_i^{[1]}\varepsilon_i\Big\},\nonumber\\
(\boldsymbol{\Phi}^{[2]})\t\boldsymbol{K}^{[2]}\dot{\boldsymbol{\Phi}}^{[2]}
														=&\sum_{j=n+1}^{n+m}\Big\{-(\Phi_j^{[2]})\t K_j^{[2]}K_j^{[2]}\Phi_j^{[2]}\nonumber\\
														&-(\Phi_j^{[2]})\t K_j^{[2]}F_j^{[2]}(F_j^{[2]})\t K_j^{[2]}\Phi_j^{[2]}\nonumber\\
														&+c_j(\Phi_j^{[2]})\t K_j^{[2]}F_j^{[2]}L_j\boldsymbol{\widetilde{\omega}}+(\Phi_j^{[2]})\t K_j^{[2]}\vartheta_j\Big\}	
\end{align}
with the $i$-th row $L_i$ and $j$-th row $L_j$ of the matrix $\mathcal L$ given in Section~\ref{Preliminary_D}.

On the other hand, it follows from $\dot{\boldsymbol{\widetilde{\omega}}}$ in Eqs.~\eqref{err_dynamic_USV2},~\eqref{err_dynamic_UAV2} that
\begin{align}
\label{V_equation1}
\boldsymbol{\widetilde{\omega}}\t C\mathcal {L}\dot{\boldsymbol{\widetilde{\omega}}}
														=&-\sum_{i=1}^{n+m}\|c_iL_i\boldsymbol{\widetilde{\omega}}\|^2 +\sum_{i=1}^{n}c_i (\Phi_i^{[1]})\t K_i^{[1]}F_i^{[1]}	L_i \boldsymbol{\widetilde{\omega}}	\nonumber\\
														   &+\sum_{j=n+1}^{n+m}c_j (\Phi_j^{[2]})\t K_j^{[2]}F_j^{[2]}L_j 	
														    \boldsymbol{\widetilde{\omega}}.									
\end{align}
Since
\begin{align}
\label{V_fact1}
&-\Big\|(\Phi_i^{[1]})\t K_i^{[1]}F_i^{[1]}\Big\|^2-\Big\|c_iL_i\boldsymbol{\widetilde{\omega}}\Big\|^2+2c_i(\Phi_i^{[1]})\t K_i^{[1]}F_i^{[1]}L_i\boldsymbol{\widetilde{\omega}}\nonumber\\
=&-\Big\|(\Phi_i^{[1]})\t K_i^{[1]}F_i^{[1]}-c_iL_i\boldsymbol{\widetilde{\omega}}\Big\|^2, i\in \mathcal V_1,\nonumber\\
&-\Big\|(\Phi_j^{[2]})\t K_j^{[2]}F_j^{[2]}\Big\|^2-\Big\|c_jL_j\boldsymbol{\widetilde{\omega}}\Big\|^2+2c_j(\Phi_j^{[1]})\t K_j^{[2]}F_j^{[2]}L_i\boldsymbol{\widetilde{\omega}}\nonumber\\
=&-\Big\|(\Phi_j^{[2]})\t K_j^{[2]}F_j^{[2]}-c_jL_j\boldsymbol{\widetilde{\omega}}\Big\|^2, j\in\mathcal V_2,
\end{align}
it follows from Eqs.~\eqref{V_equation}, \eqref{V_equation1} and \eqref{V_fact1} that $\dot{V}$ in Eqs.~\eqref{d_func_V}  can be rewritten 
\begin{align}
\label{d_func_V1}
\dot{V}=&\sum_{i=1}^n\Big\{-(\Phi_i^{[1]})\t K_i^{[1]}K_i^{[1]}\Phi_i^{[1]}+(\Phi_i^{[1]})\t K_i^{[1]}\varepsilon_i\nonumber\\
             &-\Big\|(\Phi_i^{[1]})\t K_i^{[1]}F_i^{[1]}-c_iL_i\boldsymbol{\widetilde{\omega}}\Big\|^2\Big\}\nonumber\\
             &+\sum_{j=n+1}^{n+m}\Big\{-(\Phi_j^{[2]})\t K_j^{[2]}K_j^{[2]}\Phi_j^{[2]}+(\Phi_j^{[2]})\t K_j^{[2]}\vartheta_j\nonumber\\
             &-\Big\|(\Phi_j^{[2]})\t K_j^{[2]}F_j^{[2]}-c_jL_j\boldsymbol{\widetilde{\omega}}\Big\|^2\Big\}.										
\end{align}
Since 
\begin{align}
\label{V_fact2}
&-(\Phi_i^{[1]})\t K_i^{[1]}K_i^{[1]}\Phi_i^{[1]}+(\Phi_i^{[1]})\t K_i^{[1]}\varepsilon_i\nonumber\\
\leq&-\frac{\Big\|K_i^{[1]}\Phi_i^{[1]}\Big\|^2}{2}+\frac{\|\varepsilon_i\|^2}{2}, i\in \mathcal V_1,\nonumber\\
&-(\Phi_j^{[2]})\t K_j^{[2]}K_j^{[2]}\Phi_j^{[2]}+(\Phi_j^{[2]})\t K_j^{[2]}\vartheta_j\nonumber\\
\leq&-\frac{\Big\|K_j^{[2]}\Phi_j^{[2]}\Big\|^2}{2}+\frac{\|\vartheta_j\|^2}{2}, j\in \mathcal V_2,
\end{align}
it follows from Eqs.~\eqref{d_func_V1}, \eqref{V_fact2} that
\begin{align}
\label{d_func_V2}
\dot{V}\leq&-\sum_{i=1}^n\alpha_i-\sum_{j=n+1}^{n+m}\beta_j+\gamma									
\end{align}
for 
\begin{align}
\label{replace_variable_1}
\alpha_i:=&\Big\|(\Phi_i^{[1]})\t K_i^{[1]}F_i^{[1]}-c_iL_i\boldsymbol{\widetilde{\omega}}\Big\|^2+\frac{1}{2}\Big\|K_i^{[1]}\Phi_i^{[1]}\Big\|^2,\nonumber\\
\beta_j:=&\Big\|(\Phi_j^{[2]})\t K_j^{[2]}F_j^{[2]}-c_jL_j\boldsymbol{\widetilde{\omega}}\Big\|^2+\frac{1}{2}\Big\|K_j^{[2]}\Phi_j^{[2]}\Big\|^2,\nonumber\\
\gamma:=&\sum_{i=1}^n\frac{\|\varepsilon_i\|^2}{2}+\sum_{j=n+1}^{n+m}\frac{\|\vartheta_j\|^2}{2}.
\end{align}
Due to the fact $\alpha_i\geq 0, \beta_j\geq 0, \forall i\in \mathcal V_1, j\in \mathcal V_2$ in Eq.~\eqref{replace_variable_1}, one has 
\begin{align*}
\label{integre_alpha_beta}
-\sum_{i=1}^n\int_0^t\alpha_i(s)ds-\sum_{j=n+1}^{n+m}\int_0^t\beta_j(s)ds\leq 0.
\end{align*}
Moreover, it follows from Lemma~\ref{lemma_signal_regulation} and $\varepsilon_i, \vartheta_j$ in Eqs.~\eqref{err_dynamic_USV2} and \eqref{err_dynamic_UAV2} that 
\begin{align}
\lim_{t\rightarrow\infty}\|\varepsilon_i(t)\|=0, \lim_{t\rightarrow\infty}\|\vartheta_j(t)\|=0
\end{align}
exponentially, which implies that $\gamma$ in Eq.~\eqref{replace_variable_1} satisfies 
\begin{align}
\label{conver_gamma}
\lim_{t\rightarrow\infty}\gamma(t)=0,
\end{align}
exponentially. Integrating $V(t)$ in Eq.~\eqref{d_func_V2} yields
\begin{align}
\label{bounded_V}
V(t)\leq&V(T)-\sum_{i=1}^n\int_0^t\alpha_i(s)ds-\sum_{j=n+1}^{n+m}\int_0^t\beta_j(s)ds\nonumber\\
		    &+\int_0^t\gamma(s)ds, \forall t\geq T.
\end{align}
From the fact that $\lim_{t\rightarrow\infty}\gamma(t)=0$ exponentially in~Eq.~\eqref{conver_gamma}, one has that $\int_0^t\gamma(s)ds$ is bounded, which then leads to Eq.~\eqref{integre_alpha_beta} and the boundedness of $V(T)$ (i.e., $V(t)$ in Eq.~\eqref{bounded_V} is bounded for $t\geq T$). So are error states $\Phi_i^{[1]}, \Phi_j^{[2]}, \omega_k, i\in\mathcal V_1, j\in\mathcal V_2, k\in\mathcal V$ in Eq.~\eqref{func_V}, which implies the divergence behavior has been prevented. The proof is thus completed.
\end{proof}

\begin{remark}
Since the 2D prescribed paths for the USVs implicitly contain the z-axis information of $\phi_{i,z}=0$, it thus does not influence the convergence to the coordinated navigation of the CDUS. Moreover, by introducing $\phi_{i,z}=0$, the derivative of the desired virtual coordinate of USV $i, i\in\mathcal V_1,$ is the same as the one of UAV $j, j\in\mathcal V_2$ (i.e., $\dot{\omega}_i^{\ast}=\dot{\omega}_j^{\ast}=(-1)^3, i\in\mathcal V_1, j\in\mathcal V_2$ in Eqs.~\eqref{derivative_omega_USV} and~\eqref{err_virc_UAV}), which can thus be utilized to technically guarantee the path convergence and maneuvering, and heterogeneous coordination in Lemmas~\ref{lemma_convergence_behavior} and~\ref{lemma_coordination_behavior} later.
\end{remark}

\begin{lemma}
\label{lemma_convergence_behavior}
Under Assumptions~\ref{assp_derivative}, all vehicles in the CDUS composed of  \eqref{USV_kinetic}, \eqref{UAV_kinetic}, \eqref{control_law_USV}, \eqref{control_law_UAV} approach and maneuver along the prescribed paths, i.e., 
$\lim_{t\rightarrow\infty}\Phi_{i}^{[1]}(t)=\mathbf{0}_2, \lim_{t\rightarrow\infty}\Phi_{j}^{[2]}(t)=\mathbf{0}_3, i\in\mathcal V_1, j\in\mathcal V_2$, and $\lim_{t\rightarrow\infty}$ $\dot{\omega}_i(t)=\dot{\omega}_i^{\ast}\neq0, i\in\mathcal V$.
\end{lemma}

\begin{proof}
According to \eqref{bounded_V} in Lemma~\ref{finite_time_behavior}, one has that $V(t)$ is bounded, which then follows from Eqs.~\eqref{err_dynamic_USV2},~\eqref{err_dynamic_UAV2}, \eqref{func_V} that $\Phi_i^{[1]}, \dot{\Phi}_i^{[1]}, \Phi_j^{[2]}, \dot{\Phi}_j^{[2]}, \boldsymbol{\widetilde{\omega}}, \dot{\boldsymbol{\widetilde{\omega}}}, i\in \mathcal V_1, j\in\mathcal V_2$ are all bounded as well.

Denoting $\Xi:=\sum_{i=1}^n\alpha_i+\sum_{j=n+1}^{n+m}\beta_j$ and combining with Eq.~\eqref{bounded_V} yields
\begin{align}
\label{bounded_V1}
\int_T^t\Xi(s)ds \leq&V(T)-V(t)+\int_0^t\gamma(s)ds, \forall t\geq T,
\end{align}
which implies that $\int_T^t\Xi(s)ds$ is upper bounded. From Assumption~\ref{assp_derivative}, one has that $\dot{F}_i^{[1]}, \dot{F}_j^{[2]}$ are bounded, so are the values of $(\dot{\Phi}_i^{[1]})\t K_i^{[1]}F_i^{[1]}, (\dot{\Phi}_j^{[2]})\t K_j^{[2]}F_j^{[2]}, (\Phi_i^{[1]})\t K_i^{[1]}\dot{F}_i^{[1]}$, $(\Phi_j^{[2]})\t K_j^{[2]}\dot{F}_j^{[2]}$, which implies that $\dot{\Xi}$ is bounded as well. Using Barbalat's lemma~\cite{khalil2002nonlinear}, one has that 
\begin{align*}
\lim_{t\rightarrow\infty} \Xi(t)=0,
\end{align*}
which then follows from \eqref{replace_variable_1} that $\lim_{t\rightarrow\infty}\Phi_i^{[1]}(t)=\mathbf{0}_2, \lim_{t\rightarrow\infty}\Phi_j^{[2]}(t)=\mathbf{0}_3, i\in\mathcal V_1, j\in\mathcal V_2.$
The path convergence (i.e., Objective 1 in Definition~\ref{CPF_definition}) is thus proved. Moreover, we also have
\begin{align}
\label{alpha_convergence_1}
&\lim_{t\rightarrow\infty}(\Phi_i^{[1]})\t(t) K_i^{[1]}F_i^{[1]}(t)-c_iL_i\boldsymbol{\widetilde{\omega}}(t)=0, i\in\mathcal V_1,\nonumber\\
&\lim_{t\rightarrow\infty}(\Phi_j^{[2]})\t(t) K_j^{[2]}F_j^{[2]}(t)-c_jL_j\boldsymbol{\widetilde{\omega}}(t)=0,j\in\mathcal V_2,
\end{align}
with the $i$-th row  $L_i$ and $j$-th row $L_j$ of $\mathcal L$ given in Section~\ref{Preliminary_D}.
It follows from Eqs.~\eqref{err_dynamic_USV2}, \eqref{err_dynamic_UAV2} and \eqref{alpha_convergence_1} that 
\begin{align}
\label{derivative_omega_convege}
\lim_{t\rightarrow\infty}\dot{\widetilde{\omega}}_i=\lim_{t\rightarrow\infty}\dot{\widetilde{\omega}}_j=0, i\neq j\in\mathcal V.
\end{align}
Since $\dot{\omega}_i^{\ast}=\dot{\omega}_j^{\ast}=(-1)^3, i\in\mathcal V_1, j\in\mathcal V_2,$ in \eqref{derivative_omega_USV} and \eqref{derivative_omega_UAV}, one has
\begin{align*}
\lim_{t\rightarrow\infty}\dot{\omega}_i(t)=\dot{\omega}_i^{\ast}\neq0, \forall i\in\mathcal V,
\end{align*}
which implies the path maneuvering (i.e., Objective 2 in Definition~\ref{CPF_definition}) is proved as well. The proof of Lemma~\ref{lemma_convergence_behavior} is thus completed. 
\end{proof}

\begin{lemma}
\label{lemma_coordination_behavior}
Under Assumptions~\ref{assu_topology}, all vehicles in the CDUS composed of  \eqref{USV_kinetic}, \eqref{UAV_kinetic}, \eqref{control_law_USV}, \eqref{control_law_UAV} guarantee the heterogeneous coordination via the interaction of relative virtual coordinates, i.e.,  $\lim_{t\rightarrow\infty} \omega_i(t)-\omega_j(t)=\Delta_{i,j}, j\in\mathcal N_i, i\in \mathcal V$.
\end{lemma}

\begin{proof}

On one hand, it follows from Eq.~\eqref{alpha_convergence_1} in Lemma~\ref{lemma_convergence_behavior} that $\lim_{t\rightarrow\infty}c_iL_i\boldsymbol{\widetilde{\omega}}(t)=0, \lim_{t\rightarrow\infty}c_jL_j\boldsymbol{\widetilde{\omega}}(t)=0,  i\in\mathcal V_1, j\in\mathcal V_2$, i.e., $\lim_{t\rightarrow\infty} \mathcal L\boldsymbol{\widetilde{\omega}}(t)=0 \boldsymbol{\widetilde{\omega}}=\mathbf{0}_{n+m}$,
which implies that $0$ is one of the eigenvalues of the matrix $\mathcal L$, and the vector $\boldsymbol{\widetilde{\omega}}$ is the corresponding eigenvector for $\lambda(\mathcal L)=0$.

On the other hand, recalling Assumption~\ref{assu_topology}, if $\mathcal G (\mathcal V, \mathcal E)$ contains a directed spanning tree, the matrix $\mathcal L$ contains an eigenvalue $0$ with the corresponding eigenvector $\mathbf{1}_{n+m}:=[1, \cdots, 1]\t\in\mathbb{R}^{n+m}$~\cite{ren2005consensus}.

Combining the aforementioned two claims together, one has that the vector $\boldsymbol{\widetilde{\omega}}$ is proportional to the vector $\mathbf{1}_{n+m}$,
which implies that for arbitrary two vehicles $i, j, i\neq j$, the virtual coordinate errors $\widetilde{\omega}_i, \widetilde{\omega}_j$  satisfy
\begin{align}
\label{virtual_coordinate_converge}
\lim_{t\rightarrow\infty} \widetilde{\omega}_i(t)-\widetilde{\omega}_j(t)=0, \forall i\neq j\in\mathcal V.
\end{align}
From the fact $\dot{\omega}_{i}^{\ast}=(-1)^3, i\in\mathcal V$ in \eqref{derivative_omega_USV} and \eqref{derivative_omega_UAV}, one has that the relative value between the desired virtual coordinates $\omega_i^{\ast}, \omega_j^{\ast}$ maintains a constant $\Delta_{i,j}$ all along, i.e., 
\begin{align}
\label{desired_relative_omega_val}
\omega_i^{\ast}(t)-\omega_j^{\ast}(t)=\Delta_{i,j}, \forall t>0, \forall i\neq j\in\mathcal V.
\end{align}
Then, it follows from Eqs~\eqref{virtual_coordinate_converge}, \eqref{desired_relative_omega_val} that
 \begin{align}
&\lim_{t\rightarrow\infty}\{\omega_i(t)-\omega_j(t)\}\nonumber\\
=&\lim_{t\rightarrow\infty}\{\widetilde{\omega}_i(t)-\widetilde{\omega}_j(t)+\omega_i^{\ast}(t)-\omega_j^{\ast}(t)\}\nonumber\\
=&\Delta_{i,j}, j\in\mathcal N_i, i\in \mathcal V.
\end{align}
The proof of Lemma~\ref{lemma_coordination_behavior} is thus completed.
\end{proof}

\begin{theorem}
\label{theorem_CDUS}
A CDUS composed of  \eqref{USV_kinetic}, \eqref{UAV_kinetic}, \eqref{control_law_USV}, \eqref{control_law_UAV} can achieve the coordinated navigation, i.e., both Problems~\ref{problem_1}-\ref{problem_2} are solved.
\end{theorem}

\begin{proof}
The conclusion is drawn from Lemmas~\ref{lemma_signal_regulation}-\ref{lemma_coordination_behavior} directly.
\end{proof}

\begin{remark}
For the bounded disturbances $d_i$ encountered in USV and UAV kinematics, there exist various works focusing on the disturbance observer $\widehat{d}_i$ for the compensation of $d_i$, such as the extended state observers (in abbr., ESO) and sliding mode observers (in abbr., SMO), which can estimate the disturbances in finite time~(refers to \cite{peng2020output}), i.e., $\lim_{t\rightarrow T_2}\{\widehat{d}_i(t)-d_i(t)\}=0$ with a constant time $T_2>0$. The design of such disturbance observers  has been seamlessly embedded in the low-level signal tracking regulation in Problem~\ref{problem_1}, which is less relevant in the present study  because the proposed DGVF controllers \eqref{control_law_USV} and \eqref{control_law_UAV} focus on the
upper-level desired guidance velocities in Problem~\ref{problem_2} instead.
\end{remark}

\begin{figure}[!htb]
\centering
\includegraphics[width=\hsize]{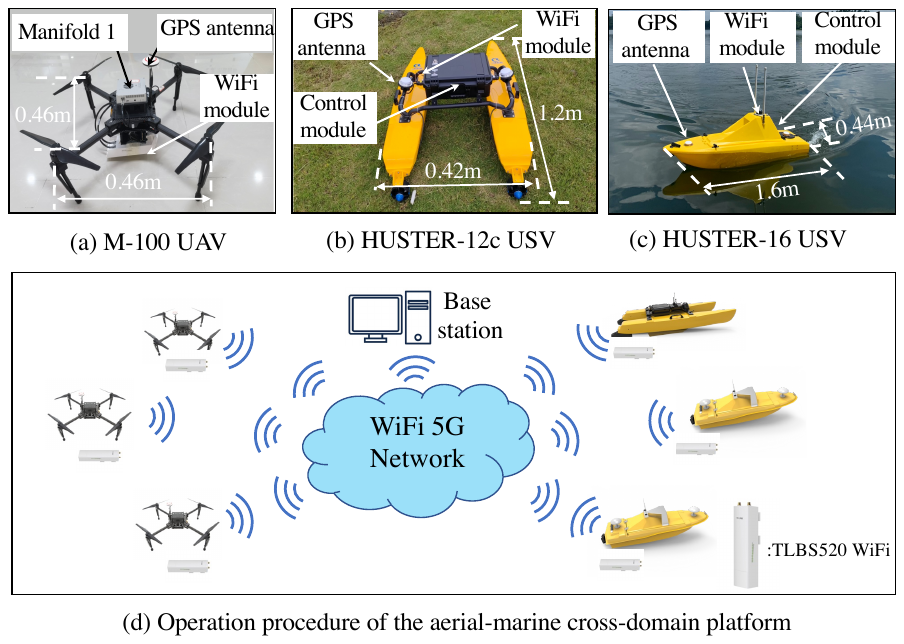}
\caption{Subfigures (a)-(c): M-100 UAV, HUSTER-12c USV, HUSTER-16 USV and the detailed components. Subfigure (d): The operation procedure of the cross-domain heterogeneous system, which consists of three M-100 UAVs, two HUSTER-16 USV, a HUSTER-12c USV and a WiFi 5G (TP-link TLBS520) wireless communication station.}
\label{CDUS_platform}
\end{figure}
\begin{figure}[!htb]
\centering
\includegraphics[width=7.9cm]{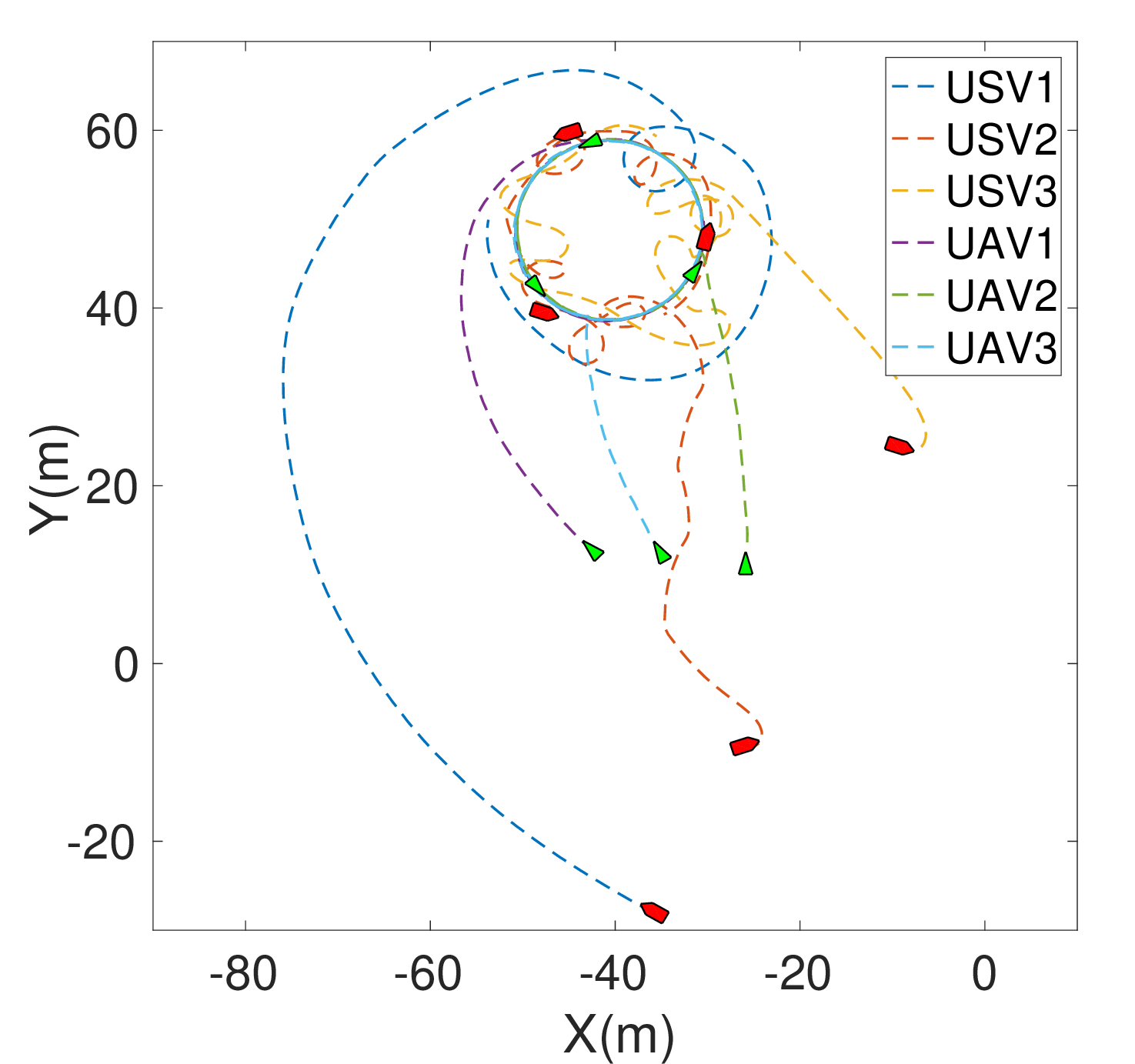}
\caption{Trajectories of the CDUS consisting of three heterogeneous USVs and three UAVs from random initial
positions to the coordinated navigation moving along a prescribed circular path from the top view. (Here, the vessel shapes denote the USVs, and the triangles the UAVs. The dashed lines are the trajectories.)}
\label{experiment_trajectory}
\end{figure}

\begin{remark}
\label{re_proximity}
The high proximity between the vehicles refers to the problem of inter-vehicle collision avoidance, which is important for real experiments. In this paper, we focus on the design of the DGVF controllers to achieve the coordination navigation of the CDUS and thus do not elaborate on details of the collision-avoidance issues. However, the designed DGVF controllers in Eqs.~\eqref{control_law_USV} and \eqref{control_law_UAV} have been seamlessly incorporated with the existing collision avoidance algorithms, such as quadratic programming (in abbr., QP) with the control barrier functions (in abbr., CBFs)~\cite{wang2017safety}. Specifically, let $u_{i,f}\in\mathbb{R}, v_{i,f}\in\mathbb{R}$ and $p_{j, x}^f\in\mathbb{R}, p_{j, y}^f\in\mathbb{R}, p_{j, z}^f\in\mathbb{R}$ be the final guidance signals of USV $i, i\in\mathcal V_1$, and UAV $j, j\in\mathcal V_2$, respectively, one can formulate the QP for the CDUS to be $\min\{(u_{i,f}-u_{i,r})^2+(v_{i,f}-v_{i,r})^2\}$ subject to $\dot{\phi}_{i,k}(q_i, q_k)\leq\phi_{i,k}(q_i, q_k), i\neq k \in\mathcal V_1$, and $\min\{(p_{j,x}^f-p_{j,x}^r)^2+(p_{j,y}^f-p_{j,y}^r)^2+(p_{j,z}^f-p_{j,z}^r)^2\}$ subject to $\dot{\phi}_{j,l}(q_j, q_l)\leq\phi_{j,l}(q_j, q_l), j\neq l \in\mathcal V_2$. Here, $\phi_{i,k}=\|q_i-q_k\|^2- R^2$ and $\phi_{j,l}=\|q_j-q_l\|^2- R^2$ are the designed CBFs with a safe radius $R\in\mathbb{R}^+$, $u_i^r, v_i^r, p_{j,x}^r, p_{j,y}^r, p_{j,z}^r$ are the DGVF controllers in Eqs.~\eqref{control_law_USV} and \eqref{control_law_UAV}, and $q_i, q_k, q_j, q_l$ are positions of USVs and UAVs given in Eqs.~\eqref{USV_kinetic} and \eqref{err_phi_USV}, respectively. For the rigorous technical analysis with additional inter-vehicle collision avoidances, it will be investigated in the future work. 
\end{remark}

\begin{figure}[!htb]
\centering
\includegraphics[width=7.45cm]{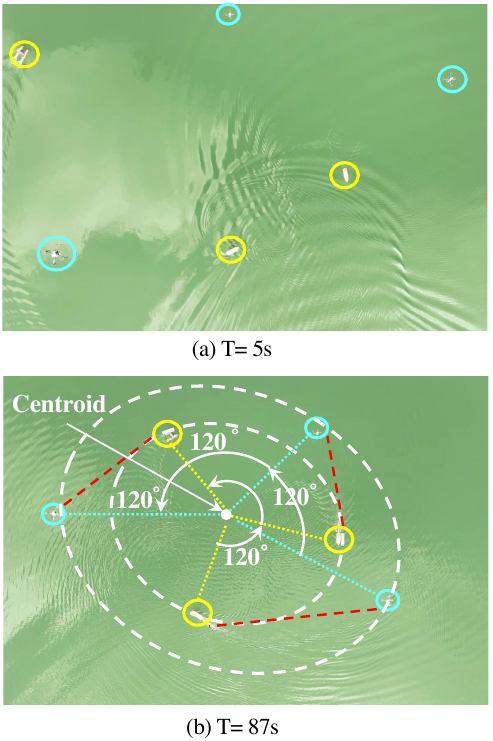}
\caption{(a) Experimental snapshot to cover all UAVs and USVs at $T=5s$. (b) Experimental snapshot of the prescribed circular-path navigation at $T=87s$. Note that the two-size circles 
are caused by different heights of the UAVs and USVs, which are in the same radius $r=10$m, and the red dashed lines denote that the corresponding USVs and UAVs are with the same x-axis and y-axis positions, but with a different height from the top view. (Here, the yellow circles are the USVs, and the blue circles the UAVs.) }
\label{experiments_snapshot}
\end{figure}

\begin{figure}[!htb]
\centering
\includegraphics[width=\hsize]{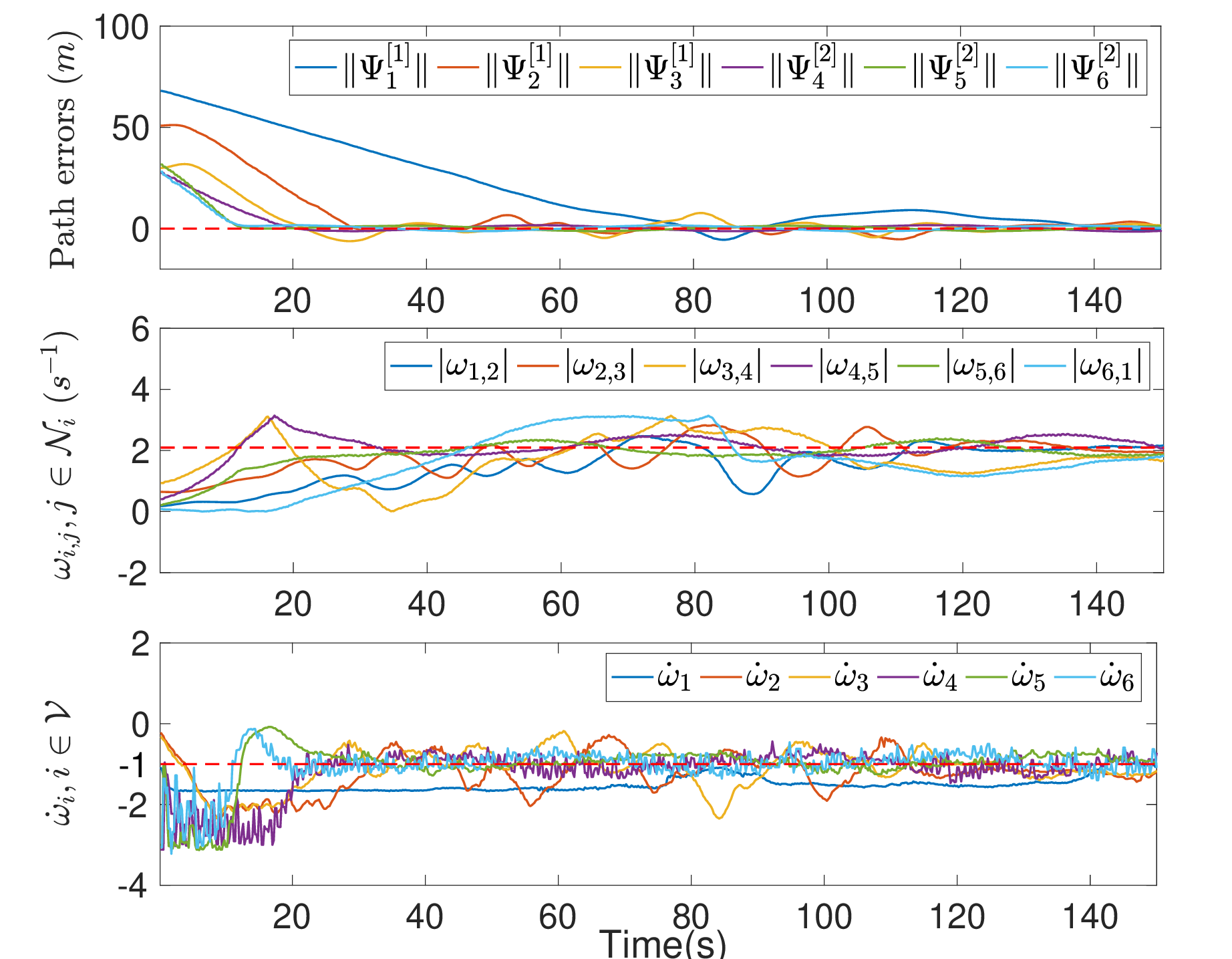}
\caption{Temporal evolution of the path-following errors $\|\Phi_i^{[1]}\|, i=1,2,$ $3,\|\Phi_j^{[2]}\|, j=4,5,6$, the inter-vehicle parametric parameters $|\omega_{1,2}|, |\omega_{2,3}|, $ $|\omega_{3,4}|, |\omega_{4,5}|, |\omega_{5,6}|, |\omega_{6,1}|$, and the derivative of the virtual coordinates $\omega_i, i=1,2,3,4,5,6$.}
\label{experiment_errors}
\end{figure}

\section{Experimental results}
In this section, we demonstrate the effectiveness of the DGVF controller in Eqs.~\eqref{control_law_USV} and \eqref{control_law_UAV} by real-lake experiments and numerical simulations.

\subsection{Cross-Domain Heterogeneous Platform}
In this subsection, we first introduce a self-established cross-domain heterogeneous platform, which consists of six unmanned vehicles, namely, three M-100 UAVs, two HUSTER-16 USVs, a HUSTER-12C USV, and a WiFi 5G wireless communication station. As shown in Fig.~\ref{CDUS_platform}, each M-100 USV is $0.46$m in length and width, which is equipped with a GPS antenna: DJI A2 GPS PRO, a control module: Manifold 1 and a WiFi module: TP-link TLBS520. 
More details of the M-100 UAV refer to \cite{zhang2021visual}.
The HUSTER-12c USV is $1.2$m in length and $0.42$m in width, and each HUSTER 16 USV is $1.6$m in length and $0.44$m in width, where both two kinds of USVs are equipped with two GPS antennas: CA-6152A, a control module: STM32F407, and the same WiFi module: TP-link TLBS520. Fig.~\ref{CDUS_platform} illustrates the coordination procedure of the cross-domain heterogeneous platform, where the UAVs and USVs localize their positions via GPS modules and interact with their neighbors via a WiFi $5$G network to generate the desired guidance signals for the coordinated navigation operations. Note that all states including trajectories, errors, etc., are transmitted and recorded by the Base station via WiFi $5$G networks.

\subsection{Experimental Results}
\label{sub_sec_exp}
Using the self-established cross-domain platform, we conduct a coordinated navigation task maneuvering along prescribed circular paths. More precisely, the prescribed circular paths for the USVs are the same and set to be $\sigma_{i,x}=-40+10\cos(\omega_i), \sigma_{i,y}=50+10\sin(\omega_i), i\in\mathcal V_1$, whereas the corresponding prescribed circular paths for UAVs are $\sigma_{j,x}=-40+10\cos(\omega_j), \sigma_{j,y}=50+10\sin(\omega_j), \sigma_{j,z}=20, j\in\mathcal V_2$. The control gains in controller \eqref{control_law_USV} and \eqref{control_law_UAV} are set to be $k_{i,x} =k_{i,y}=1.5, c_i=2, i\in\mathcal V_1$, and $k_{j,x} =k_{j,y}=k_{j,z}=1.5, c_j=2, j\in\mathcal V_2$. Moreover, we conduct three steps to prevent high proximity among the vehicles. Firstly, the initial positions of the USVs and UAVs are distrbuted at a safe distance to avoid high proximity from the beginning. Secondly, by properly setting the desired parametric displacements to be $\Delta_{i,i+1}=2\pi/3$ for $ i=1,2,3,4,5,$ ($i+1~\mbox{mod}~6$, if $i=6$), we can avoid high proximity once if the coordinated navigation of the CDUS is achieved. Thirdly, we embed the aforementioned collision avoidance algorithms of CBFs in Remark~\ref{re_proximity} to avoid high proximity during the coordinated navigation process.
Under Assumption~\ref{assu_topology}, the communication topology $\mathcal G (\mathcal V, \mathcal E)$ for the six vehicles is set to be a cyclic connected graph for convenience.

\begin{figure}[!htb]
\centering
\includegraphics[width=8cm]{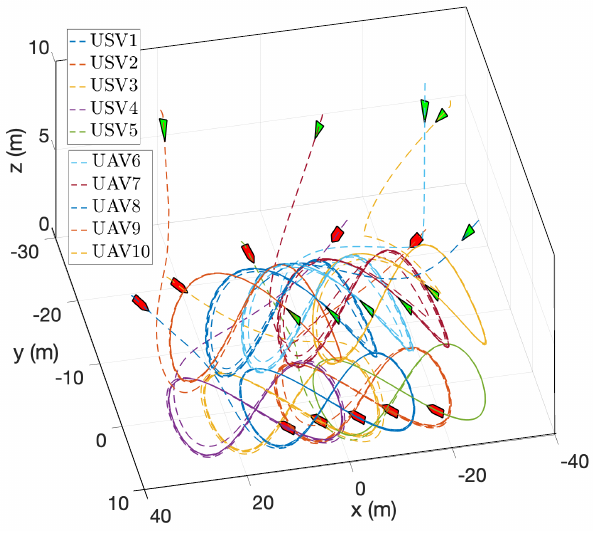}
\caption{Trajectories of the CDUS consisting of 5 USVs and 5 UAVs from random initial
positions to the coordinated navigation moving along the prescribed self-intersecting paths in 3D. (Here, the vessel shapes denote the USVs, and the triangles the UAVs. The dashed lines are the trajectories.)}
\label{self_intersecting_trajectory}
\end{figure}

\begin{figure}[!htb]
\centering
\includegraphics[width=\hsize]{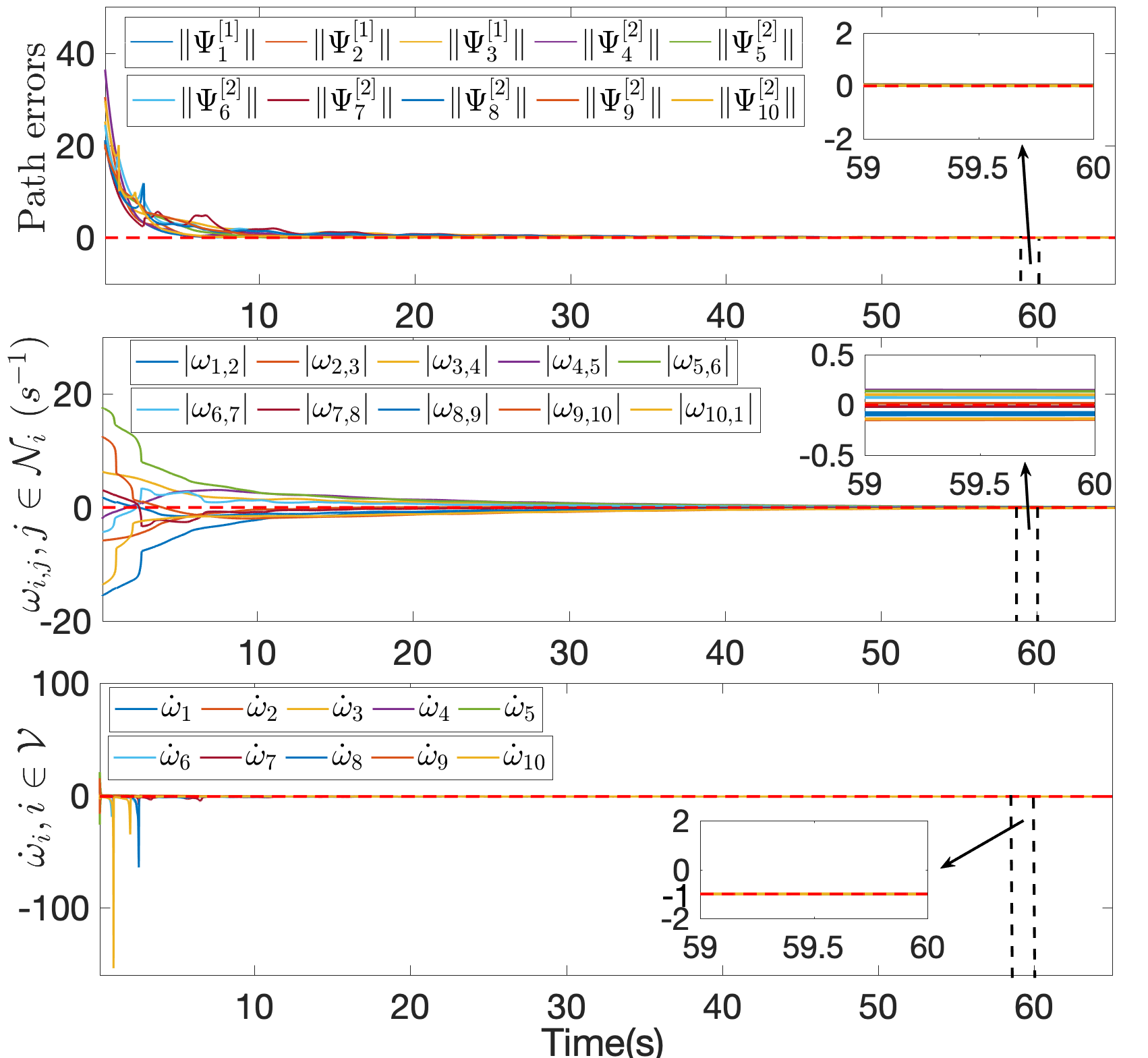}
\caption{Temporal evolution of the path-following errors $\|\Phi_i^{[1]}\|, i\in\mathcal V_1, \|\Phi_j^{[2]}\|, j\in\mathcal V_2$, the inter-vehicle parametric parameters $|\omega_{i,i+1}|, i\in \mathcal V_1 \cup \mathcal V_2$ ($i+1=1$ if $i=10$), and the derivative of the virtual coordinates $\omega_i, i\in\mathcal V_1\cup\mathcal V_2$.}
\label{self_intersecting_states}
\end{figure}

Fig.~\ref{experiment_trajectory} illustrates the trajectories of three USVs and three UAVs from random initial positions to the final coordinated circular navigation formation.
As shown in Fig.~\ref{experiments_snapshot}, the experimental snapshots validate that the coordinated circular navigation of the CDUS is finally achieved as well, where the USVs and UAVs converge to the desired parametric displacements from a top view of the camera drone. More precisely, it is observed in Fig.~\ref{experiment_errors} that the path-following errors converge to be around zeros at $t=87$ seconds, i.e., $\lim_{t\rightarrow\infty}\|\Phi_i^{[1]}(t)\|=0, i=1,2,3, \lim_{t\rightarrow\infty}\|\Phi_j^{[2]}\|=0, j=4,5,6$, which fulfills Objective 1 of the path convergence in Definition~\ref{CPF_definition}. The inter-vehicle parametric parameters converge to be around $2\pi/3$ at $t=142$ seconds, i.e., $\lim_{t\rightarrow\infty} \omega_{i,j}(t)=\Delta_{i,j}$, which indicates the heterogeneous coordination in Definition~\ref{CPF_definition}. Fig.~\ref{experiment_errors} also exhibits that the derivative of the virtual coordinates $\omega_i, i\in\mathcal V$ show small oscillations around $-1$, which can be accepted in practice and thus verify the path maneuvering along the prescribed circular paths. Accordingly, the effectiveness of the desired guidance velocities in Eqs.~\eqref{control_law_USV},~\eqref{control_law_UAV} and Theorem~\ref{theorem_CDUS} in practice is thus verified.

\subsection{Cross-Domain Simulations for Self-Intersecting Paths}
To verify the DGVF controller \eqref{control_law_USV} and \eqref{control_law_UAV} in more complex situations, we consider a CDUS consisting of $5$ USVs (i.e., $\mathcal V_1=\{1,2,3,4,5\}$) and $5$ UAVs (i.e., $\mathcal V_2=\{6,7,8,9,10\}$) which performs  the coordinated navigation whereas maneuvering along prescribed self-intersecting paths. More precisely, the self-intersecting paths for the USVs are set to be $\sigma_{i,x}=16\cos(0.5\omega_i), \sigma_{i,y}=6\cos(\omega_i+\pi/2)-d_{o}^i, i\in\mathcal V_1$ with the offsets $d_{o}^1=0, d_{o}^2=-7, d_{o}^3=7, d_{o}^4=14, d_{o}^5=-14$, whereas the prescribed paths for UAVs are set to be $\sigma_{j,x}=16\cos(0.5\omega_j), \sigma_{j,y}=6\cos(\omega_j+\pi/2)-d_{\it o}^j,  \sigma_{j,z}=-2\cos(\omega_j), j\in\mathcal V_2,$ with the offsets $d_{o}^6=0, d_{o}^7=-7, d_{o}^8=7, d_{o}^9=14, d_{o}^{10}=-14$. And the desired parametric displacements are specified to be $\Delta_{i,i+1}=0, i\in \mathcal V_1 \cup \mathcal V_2$ ($i+1~\mbox{mod}~10$, if $i=10$). The communication topology satisfying Assumption~\ref{assu_topology} is set to be the same as the one in Section~\ref{sub_sec_exp}.

As shown in Fig.~\ref{self_intersecting_trajectory}, five USVs (vessel shapes) and five UAVs (triangle shapes) from random initial positions finally converge to and maneuver along their prescribed self-intersecting paths whereas coordinating a formation. We also illustrate the detailed state evolution of the cross-domain navigation process. It is observed in Fig.~\ref{self_intersecting_states} that the path-following errors converge to be around zeros at $t=30$ seconds, i.e., $\lim_{t\rightarrow\infty}\|\Phi_i^{[1]}(t)\|=0, i\in\mathcal V_1, \lim_{t\rightarrow\infty}\|\Phi_j^{[2]}\|=0, j\in\mathcal V_2$, which fulfills Objective 1 in Definition~\ref{CPF_definition}. As for the evolution of inter-vehicle parametric parameters, Fig.~\ref{self_intersecting_states} illustrates that $|\omega_{i,i+1}|, i\in \mathcal V_1 \cup \mathcal V_2$ ($i+1=1$ if $i=10$), converge to be around zeros at $t=50$ seconds, i.e., $\lim_{t\rightarrow\infty} \omega_{i,j}(t)=\Delta_{i,j}$, which demonstrates the heterogeneous coordination in Objective 2 of Definition~\ref{CPF_definition}. Moreover, the derivative of the virtual coordinates $\omega_i, i\in\mathcal V_1\cup\mathcal V_2$ converge to be $-1$ as well, which indicates that path maneuvering in Objective 3 of Definition~\ref{CPF_definition}. Therefore, the feasibility of Theorem~\ref{theorem_CDUS} and the effectiveness of the present DGVF controller \eqref{control_law_USV} and \eqref{control_law_UAV} are both verified.

\section{Conclusion}
In this paper, we have presented a DGVF controller for heterogeneous CDUSs such that the aerial-marine UAV-USV group achieves coordinated navigation and maneuvers along their prescribed paths. The proposed DGVF controller consists of two items in a hierarchical control architecture, i.e., an upper-level heterogeneous guidance velocity controller to govern all vehicles to converge to and maneuver along the prescribed paths and coordinate their formation, and a lower-level signal tracking regulator  to track the corresponding desired guidance signals. The communication, sensor, and computational costs in complex heterogeneous coordination among UAVs and USVs can be reduced by only exchanging virtual coordinate scalars with neighbor vehicles. Significantly, sufficient conditions assuring the asymptotical convergence of the closed-loop system have been provided as well. Finally, the effectiveness of the proposed method has been verified by both numerical simulations and real-lake-based experiments with three M-100 UAVs, two HUSTER-16 USVs, a HUSTER-12C USV, and a WiFi 5G wireless communication station. Future work will focus on coordinated navigation control for the CDUS with strongly heterogeneous dynamics and in obstacle-dense environments. Additionally, we aim to utilize this kind of coordination technique to develop a cross-domain coordination coverage controller in various real-world applications.

\section{Acknowledgments}
The authors would like to thank Yifeng Tang and Zhihui Huang from Huazhong University of Science and Technology for their help in the setup of real-lake experiments.

\bibliographystyle{IEEEtran}
\bibliography{IEEEabrv,ref}

\end{document}